\documentclass[review,5p]{elsarticle}

\usepackage{lineno}
\usepackage{hyperref}
\journal{Neural Networks}

\usepackage{multirow}
\usepackage{amsmath}
\usepackage{amsthm}
\usepackage{amssymb}
\usepackage{mathtools}
\usepackage{booktabs}
\usepackage{xcolor}
\usepackage{nicefrac} 
\usepackage{multirow}
\usepackage{tikz}
\usetikzlibrary{arrows}
\usepackage{enumitem}
\usepackage{algorithm}
\usepackage{algorithmic}
\usepackage{bm}
\usepackage{graphicx}
\usepackage{subcaption}
\usepackage{caption}

\newcommand{\vect}[1]{\mathbf{#1}}

\newtheorem{thm}{Theorem}[]
\newtheorem{thmA}{Theorem}[]
\newtheorem{lem}{Lemma}[]

\newtheorem{prop}{Proposition}[]

\newtheorem{defn}{Definition}[]

\DeclareMathOperator*{\argmax}{arg\,max}
\DeclareMathOperator*{\argmin}{arg\,min}

\usepackage{soul}
\DeclareRobustCommand{\cc}[1]{{#1}}
\soulregister{\cc}{1}

\begin{document}

\begin{frontmatter}

\title{Adversarial Parameter Defense by Multi-Step Risk Minimization}

\author[address1]{Zhiyuan Zhang}
\ead{zzy1210@pku.edu.cn}
\author[address2]{Ruixuan Luo}
\ead{luoruixuan97@pku.edu.cn}
\author[address1]{Xuancheng Ren}
\ead{renxc@pku.edu.cn}
\author[address3,address1]{Qi Su}
\ead{sukia@pku.edu.cn}
\author[address4]{Liangyou Li}
\ead{liliangyou@huawei.com}
\author[address1,address2]{Xu Sun}
\ead{xusun@pku.edu.cn}

\address[address1]{MOE Key Laboratory of Computational Linguistics, School of EECS, Peking University, Beijing, China.}
\address[address2]{Center for Data Science, Peking University, Beijing, China.}
\address[address3]{School of Foreign Languages, Peking University, Beijing, China.}
\address[address4]{Huawei Noah’s Ark Lab, Hong Kong, China.}

\begin{abstract}
Previous studies demonstrate DNNs' vulnerability to adversarial examples and adversarial training can establish a defense to adversarial examples. In addition, recent studies show that deep neural networks also exhibit vulnerability to parameter corruptions. The vulnerability of model parameters is of crucial value to the study of model robustness and generalization. In this work, we introduce the concept of parameter corruption and propose to leverage the loss change indicators for measuring the flatness of the loss basin and the parameter robustness of neural network parameters. On such basis, we analyze parameter corruptions and propose the multi-step adversarial corruption algorithm. To enhance neural networks, we propose the adversarial parameter defense algorithm that minimizes the average risk of multiple adversarial parameter corruptions. Experimental results show that the proposed algorithm can improve both the parameter robustness and accuracy of neural networks.
\end{abstract}
\begin{keyword}
Vulnerability of Deep Neural Networks, Parameter Corruption, Adversarial Parameter Defense
\end{keyword}
\end{frontmatter}

\section{Introduction}

Deep neural networks (DNNs) have made striking breakthroughs across many application domains, such as computer vision (CV)~\citep{resnet}, natural language processing (NLP)~\citep{transformer}, and speech recognition~\citep{Speech_Recognition_seq2seq}. Despite the promising performance of DNNs, DNNs are found vulnerable to adversarial examples~\citep{Explaining_and_Harnessing_Adversarial_Examples,Adversarial_examples_in_the_physical_world,Deepfool}, i.e., simple perturbations to the input data can mislead models substantially. Adversarial training~\citep{Explaining_and_Harnessing_Adversarial_Examples,Towards_Evaluating_the_Robustness_of_Neural_Networks,YOPO,freeLB,Unified-min-max} is conducted to enhance the robustness and accuracy of deep neural networks, making it more applicable in real-world scenarios. Besides adversarial examples, parameter corruptions can also threaten neural networks. For neural networks deployed on electronic computers, parameter corruptions may occur in the forms of training data poisoning~\citep{badnet, backdoor1, backdoor2, backdoor-Bert}, bit flipping~\citep{TBT}, compression~\citep{Stronger_Generalization_Compression} or parameter quantization~\citep{Data-Free-Quant,tensorRT,parameter_L1}. For neural networks deployed in physical devices\cite{optical-NN, survey_hardware,Efficient_FPGA-based,digital_Hardware,amplifier_based_MOS_Analog_NN,Scalable_NN_on_chip}, parameter corruptions occur as hardware deterioration and background noise. Our previous work~\citep{attack-paper} shows that deep neural networks are severely affected by adversarial parameter corruptions and proposes to probe the robustness of different parameters via parameter corruption. In this work, we further analyze parameter corruption and study \textit{adversarial parameter defense}.

To evaluate the parameter robustness, we propose indicators of measuring the loss change caused by parameter corruptions. On its basis, we analyze the distribution of the random parameter corruptions and propose the multi-step adversarial corruption algorithm. Intuitively, the loss change shows the flatness of the loss surface in the neighborhood of the current parameter, as illustrated in Figure~\ref{fig:loss_curve}. Here point $A$ is a flat minimum and point $B$ is a sharp minimum. Traditional learning algorithms focus on obtaining lower loss, which means generally the parameters at point $B$ are preferred. However, point $A$ is a flat minimum and has better parameter robustness. Recent studies~\citep{On_Large-Batch_Training,Entropy-SGD} indicate that flat minima tend to have better generalization ability. Therefore, point $A$ is a better choice to gain both better parameter robustness and better generalization ability.

\begin{figure}[!t]
\centering
\includegraphics[height=8\baselineskip]{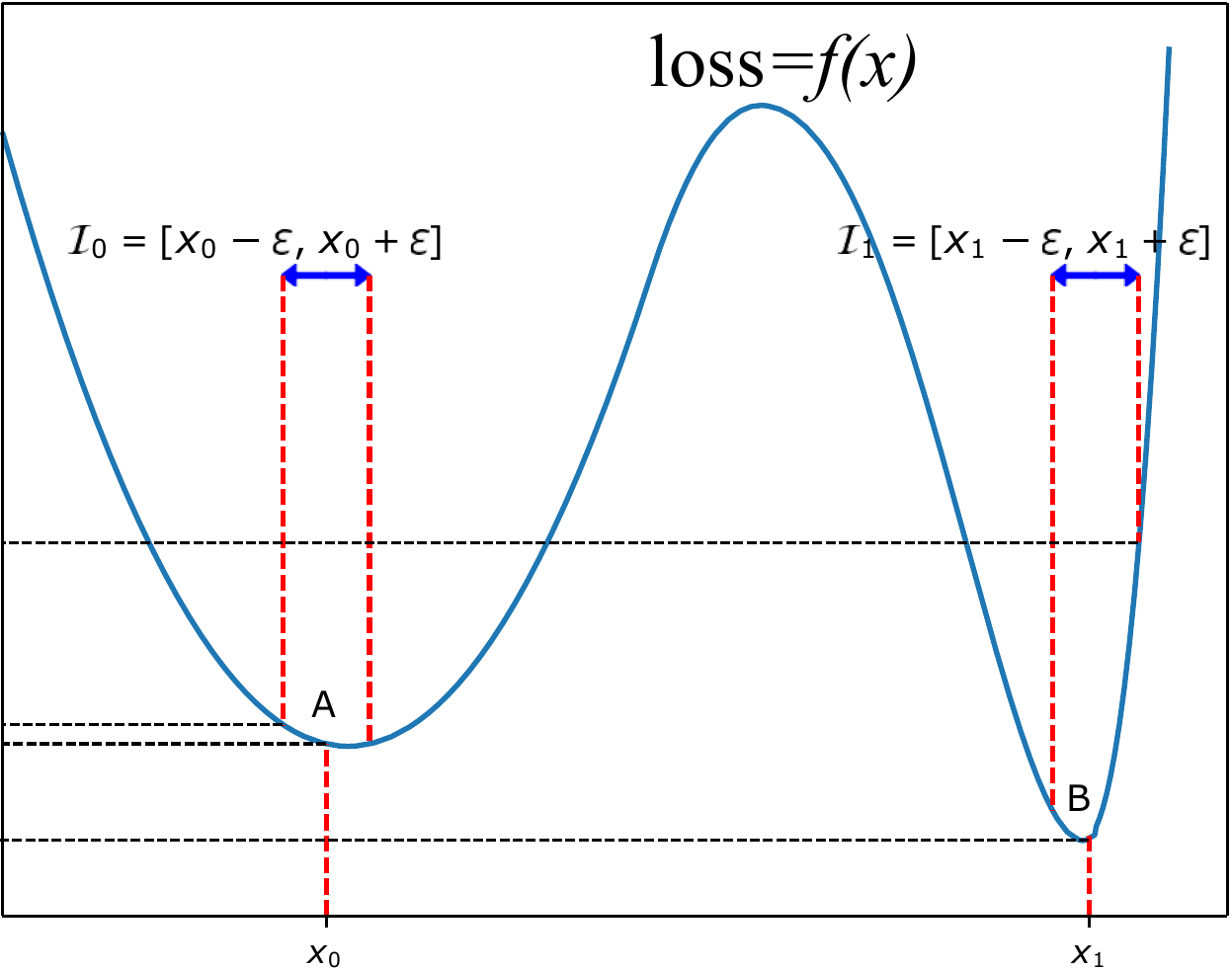}
\caption{In the illustration of the loss function from our previous work~\citep{attack-paper}, traditional optimizer prefers $B$ with the lower loss. However, point $A$ is a flat minimum and has a better parameter robustness as $\max_{x\in\mathcal{I}_0}(f(x)-f(x_0))< \max_{x\in\mathcal{I}_1}(f(x)-f(x_1))$. Therefore, point $A$ is a better choice to gain both better parameter robustness and better generalization ability.}
\label{fig:loss_curve}
\end{figure}

Loss change indicators can reveal the vulnerability of neural network parameters, which leads to poor robustness against parameter corruption and harms the generalization performance. Therefore, we propose to drive the parameters from the area with steep surroundings with the aim to improve both the accuracy and the parameter robustness of models. Figure~\ref{fig:error} illustrates our motivation. 
The goal of the algorithm is to minimize the risk of adversarial parameter corruption so as to ensure a stable accuracy even with parameter corruptions. To this end, we propose a novel approach for minimizing the risk based on multi-step adversarial parameter corruptions. A single parameter corruption is decomposed into multiple steps so that for each step, the norm of the perturbation is constrained, facilitating a more accurate gradient-based estimation. By incorporating the risk estimation into the adversarial training, we can effectively defend the neural network against parameter corruption, which may result in a better performance in terms of both robustness and accuracy. We also provide a theoretical analysis of the relation between the indicators related to the parameter defense algorithm and the generalization error bound. However, we find that not all parts of parameters are enhanced by the proposed defense algorithm, which means defending the whole model is not proper for improving the robustness of specific parameters in neural networks. To solve this issue, we further propose a localized version of the parameter defense method, considering corruptions only to the target parameters.

\begin{figure}[!t]
  \begin{minipage}[c]{\linewidth}
    \includegraphics[width=0.95\linewidth]{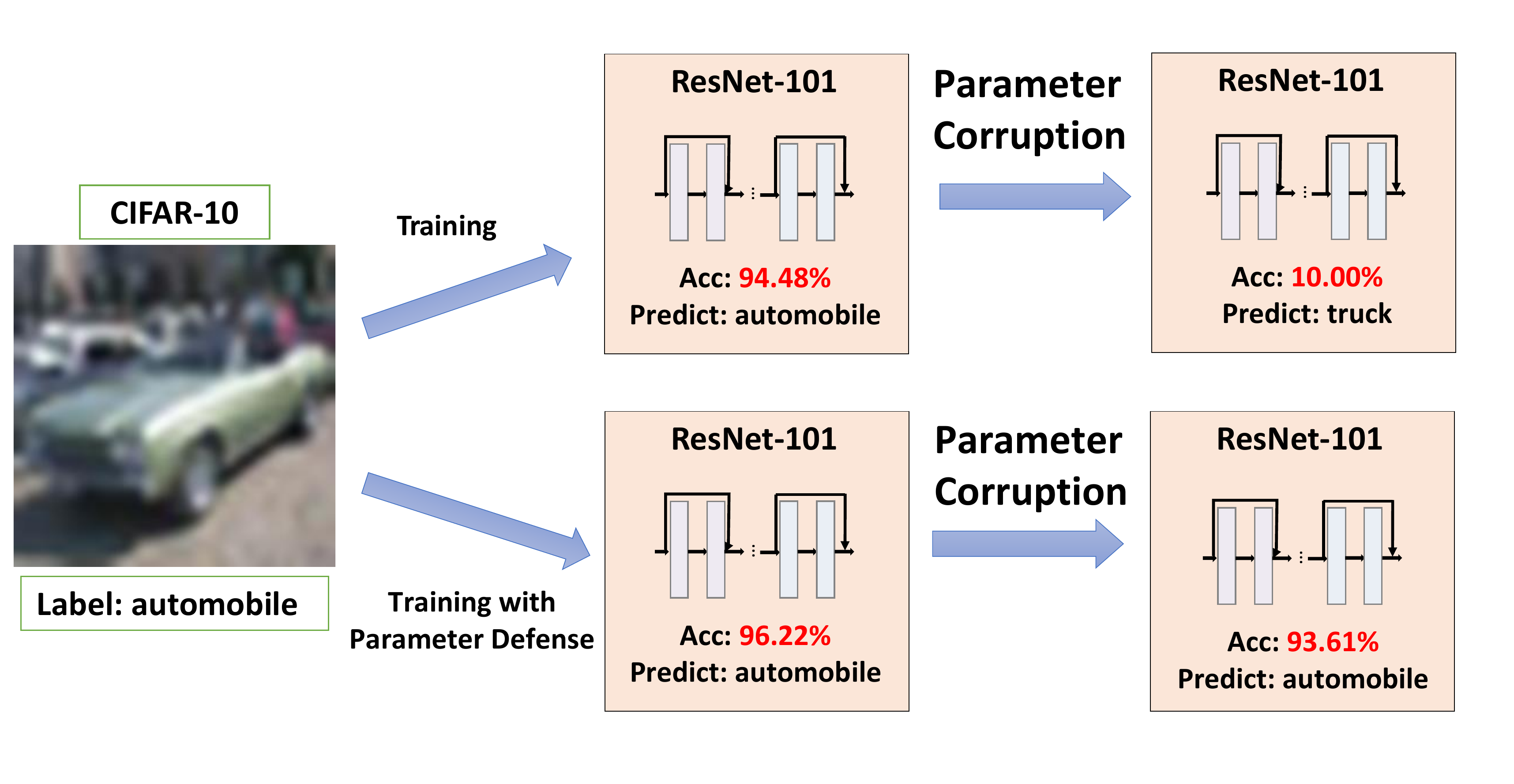}
  \end{minipage}\hfill
  \begin{minipage}[c]{\linewidth}
    \caption{Neural networks are sensitive to adversarial parameter corruptions. The proposed parameter defense algorithm can improve both the accuracy and the robustness of neural networks. The accuracy is measured on the test set and the robustness is measured as the accuracy after parameter corruptions ($\epsilon$=0.0005, $L_{+\infty}$, $n$=100).
    } \label{fig:error}
  \end{minipage}
\end{figure}

Our contributions are summarized as follows:
\begin{itemize}
\item To understand the parameter vulnerability of deep neural networks, which is fundamentally related to model robustness and generalization, we introduce the concept of parameter corruption and propose to leverage the loss change indicators. On its basis, we analyze the distribution of the random parameter corruptions and propose the multi-step adversarial corruption algorithm. 
\item  We propose the adversarial parameter defense algorithm, which minimizes the risk of adversarial parameter corruption. To estimate such risk, we propose a multi-step adversarial corruption algorithm using consecutive virtual parameter perturbations. Experimental results show that our defense algorithm can improve both the accuracy and robustness with multiple parameter corruption methods.

\end{itemize}

\section{Parameter Corruption}
\label{sec:method}

In this section, we introduce the problem of parameter corruption and the proposed indicators. Then, we analyze the distribution of the random parameter corruptions and propose the multi-step adversarial corruption algorithm.

\subsection{Notations and Definitions}

Before delving into the specifics, we first introduce our notations. Let $\vect{w}\in\mathbb{R}^k$ denote the vector of $k$ parameters allowed to be corrupted and ${\bm\theta}$ denotes the parameters non-corrupted. $\mathcal{L}(\vect{w}; \mathcal{D})$ denote the loss function on the dataset $\mathcal{D}$, regarding the specific parameter subspace $\vect{w}$. The loss function can also be written as $\mathcal{L}(\vect{w}; (\vect{x}, y))$ or $\mathcal{L}((\vect{w}, {\bm\theta}); \mathcal{D})$, where $(\vect{x}, y)$ denotes a data instance. Taking a $k$-dimensional subspace allows a more general analysis on a specific group of parameters.

To expose the vulnerability of parameters, we propose to adopt the approach of parameter corruption. To formally analyze its effect on neural networks, we formulate the parameter corruption as a small perturbation $\vect{a}\in\mathbb{R}^k$ to the parameter vector $\vect{w}$. The corrupted parameter is $\vect{w}+\vect{a}$. The corruption is specified by a constraint set.

\begin{defn}[Corruption Constraint]
The corruption constraint is specified by the set:
\begin{align}
S=\{\vect{a}:\|\vect{a}\|_p\le \epsilon \text{ and }\|\vect{a}\|_0\le n\},
\end{align}
where $\|\cdot\|_0$ denotes the number of non-zero elements in a vector and $1\le n\le k$ denotes the maximum number of corrupted parameters. $\epsilon$ is a small positive real number and $\|\cdot\|_p$ denotes the $L_p$-norm where $p\ge 1$ such that $\|\cdot\|_p$ is a valid distance. The corruption constraint can also be specified by $S=\{\vect{a}:\|\vect{a}\|_p= \epsilon \text{ and }\|\vect{a}\|_0\le n\}$ occasionally.
\end{defn}

For example, $S=\{\vect{a}:\|\vect{a}\|_2\le \epsilon\}$ specifies that the parameter corruption should be in a $k$-dimensional hyperball with a radius of $\epsilon$, where the number of corrupted parameters is not limited and $n=k$. 

Suppose $\Delta\mathcal{L}(\vect{w}, \vect{a}; \mathcal{D})=\mathcal{L}(\vect{w}+\vect{a}; \mathcal{D})-\mathcal{L}(\vect{w}; \mathcal{D})$ denotes the loss change after the parameter corruption. To evaluate the effect of parameter corruption, we propose the average loss change indicator and the maximum loss change indicator under the corruption constraints. The optimal parameter corruption is defined accordingly, which considers the worst-case scenario.

\begin{defn}[Indicators and Optimal Parameter Corruption]
The average loss change indicator $\Delta_\text{ave}\mathcal{L}(\vect{w}, S, \mathcal{D})$, the maximum loss change indicator $\Delta_\text{max}\mathcal{L}(\vect{w}, S, \mathcal{D})$, and the optimal parameter corruption $\vect{a}^*$ are defined as:
\begin{align}
\Delta_\text{ave}\mathcal{L}(\vect{w}, S, \mathcal{D})=\mathbb{E}_{\vect{a}\in S}[\Delta\mathcal{L}(\vect{w}, \vect{a}, \mathcal{D})], \\
\Delta_\text{max}\mathcal{L}(\vect{w}, S, \mathcal{D})=\max\limits_{\vect{a}\in S}\Delta\mathcal{L}(\vect{w}, \vect{a}, \mathcal{D}), \\
\vect{a}^*=\argmax_{\vect{a}\in S}\Delta\mathcal{L}(\vect{w}, \vect{a}, \mathcal{D}).
\end{align}
\end{defn}

Let $\vect{g}$ denote $\nicefrac{\partial\mathcal{L}(\vect{w}; \mathcal{D})}{\partial \vect{w}}$ and $\textbf{H}$ denote the Hessian matrix; suppose $\|\vect{g}\|_2=G>0$. Using the second-order Taylor expansion, we estimate the loss change and the proposed indicators:
\begin{equation}
\Delta\mathcal{L}(\vect{w},\vect{a}; \mathcal{D}) = \vect{a}^\text{T}\vect{g}+\frac{1}{2}\vect{a}^\text{T}\textbf{H}\vect{a}+o(\epsilon^2)=f(\vect{a})+o(\epsilon).
\end{equation}

Here, $f(\vect{a})=\vect{a}^\text{T}\vect{g}$ is a first-order estimation of the loss change $\Delta\mathcal{L}(\vect{w}, \vect{a}; \mathcal{D})$ and meanwhile the inner product of the parameter corruption $\vect{a}$ and the gradient $\vect{g}$, based on which, gradient-based corruption algorithm estimates the the optimal parameter corruption by maximizing the alternative inner product instead of the initial loss change.

We first analyze the random parameter corruption and the gradient-based corruption to understand the effect of parameter corruption. Then we propose the multi-step adversarial corruption algorithm. The detailed theoretical analysis and proofs are in Appendix D.

\subsection{Analysis of Random Corruption}

We first analyze the random case. We discuss the characteristics of the loss change caused by random corruption under a representative corruption constraint in Theorem~\ref{thm:random}. 

\begin{thm}[Distribution of Random Corruption]
\label{thm:random}
Given the constraint set  $S=\{\vect{a}:\|\vect{a}\|_2= \epsilon\}$ and a generated random corruption $\vect{\tilde a}$, which in turn obeys a uniform distribution on $\|\vect{\tilde a}\|_2=\epsilon$. The estimation of $\Delta_\text{ave}\mathcal{L}(\vect{w}, S, \mathcal{D})$ and $\Delta_\text{max}\mathcal{L}(\vect{w}, S, \mathcal{D})$ are:
\begin{align}
\Delta_\text{ave}\mathcal{L}(\vect{w}, S, \mathcal{D})&=\frac{tr(\textbf{H})}{2k}\epsilon^2+o(\epsilon^2),\\
\Delta_\text{max}\mathcal{L}(\vect{w}, S, \mathcal{D})&=\epsilon G+o(\epsilon) \label{eq:expectation}.
\end{align}
Define $\eta=\nicefrac{|\vect{\tilde a}^\text{T}\vect{g}|}{\epsilon G}$, which is a first-order estimation of  $\nicefrac{|\Delta\mathcal{L}(\vect{w}, \vect{\tilde a}, \mathcal{D})| }{\Delta_\text{max}\mathcal{L}(\vect{w}, S, \mathcal{D})}$ and $\eta\in [0, 1]$, then the probability density function $p_\eta(x)$ of $\eta$ and the cumulative density $P(\eta \le x)$ function of $\eta$ are:
\begin{align}
p_{\eta}(x)&=\frac{2\Gamma(\frac{k}{2})}{\sqrt{\pi}\Gamma(\frac{k-1}{2})}(1-x^2)^{\frac{k-3}{2}}, \label{equ:random_destiny1}\\
P(\eta \le x)&=\frac{2xF_1(\frac{1}{2}, \frac{3-k}{2};\frac{3}{2}; x^2)}{B(\frac{k-1}{2}, \frac{1}{2})}, \label{equ:random_destiny2}
\end{align}
where $k$ denotes the number of corrupted parameters, and $\Gamma(\cdot)$, $B(\cdot,\cdot)$ and $F_1(\cdot,\cdot;\cdot;\cdot)$ denote the gamma function, beta function and hyper-geometric function, respectively.
\end{thm}

Theorem~\ref{thm:random} states that the average loss change indicator $\Delta_\text{ave}\mathcal{L}(\vect{w}, S, \mathcal{D})$ is an infinitesimal of higher order compared to the maximum loss change indicator $\Delta_\text{max}\mathcal{L}(\vect{w}, S, \mathcal{D})$ when $\epsilon$ approaches $0$. In addition, it is unlikely for multiple random trials to induce the optimal loss change. For a deep neural network, the number of corrupted parameters can be considerably large. According to Eq.(\ref{equ:random_destiny1}), $\eta$ will be concentrated near $0$. Thus, theoretically, it is not generally possible for the random corruption to cause substantial loss changes in this circumstance, making it ineffective in finding the vulnerability. We should consider more effective algorithms to parameter corruption.

\subsection{Gradient-based Corruption}

To arrive at the optimal parameter corruption that renders a fast and accurate estimation of the proposed maximum loss change indicator, we further propose a gradient-based method by maximizing the first-order estimation $f(\vect{a})=\vect{a}^\text{T}\vect{g}$ of the indicator. The gradient-based parameter corruption $\vect{\hat a}$ in $S$ has a closed-form solution:
\begin{align}
\vect{\hat a}=\argmax_{\vect{a}\in S}\vect{a}^\text{T}\vect{g} &= \epsilon\left(\text{sgn}(\vect{h})\odot\frac{|\vect{h}|^\frac{1}{p-1}}{\||\vect{h}|^\frac{1}{p-1}\|_p}\right); \label{equ:proposed}\\
f(\vect{\hat a})&=\vect{\hat a}^\text{T}\vect{g}=\epsilon\|\vect{h}\|_{\frac{p}{p-1}};
\end{align}
where $\vect{h}=\text{top}_n(\vect{r})$. The $\text{top}_n(\vect{v})$ function retains top-$n$ magnitude of all $|\vect{v}|$ dimensions and set other dimensions to $0$, $\text{sgn}(\cdot)$ denotes the signum function, $|\cdot|$ denotes the point-wise absolute function, and $(\cdot)^\alpha$ denotes the point-wise $\alpha$-power function.

The error bound of the gradient-based parameter corruption is described in Theorem~\ref{thm:bound}.

\begin{thm}[Error Bound of the Gradient-Based Estimation]
\label{thm:bound}
Suppose $\mathcal{L}(\vect{w};\mathcal{D})$ is convex and $L$-smooth with respect to $\vect{w}$ in the subspace $\{\vect{w}+\vect{a}:\vect{a}\in S\}$, where $S=\{\vect{a}:\|\vect{a}\|_p=\epsilon\text{ and }\|\vect{a}\|_0\le n\}$.\footnote{Note that $\mathcal{L}$ is only required to be convex and $L$-smooth in a neighbourhood of $\vect{w}$, instead of the entire $\mathbb{R}^k$.} Suppose $\vect{a^*}$ and $\vect{\hat a}$ are the optimal corruption and the gradient-based corruption in $S$ respectively. $\|\vect{g}\|_2=G>0$. It is easy to verify that $\mathcal{L}(\vect{w}+\vect{a^*};\mathcal{D})\ge \mathcal{L}(\vect{w+\vect{\hat a}};\mathcal{D})>\mathcal{L}(\vect{w};\mathcal{D})$ . It can be proved that the loss change of the gradient-based corruption is the same order infinitesimal of that of the optimal parameter corruption:
\begin{equation}
\frac{\Delta_\text{max}\mathcal{L}(\vect{w}, S; \mathcal{D})}{\Delta\mathcal{L}(\vect{w}, \vect{\hat a}; \mathcal{D})}=1+O\left(\frac{Ln^{g(p)}\sqrt{k}\epsilon}{G}\right);
\label{eq:bound}
\end{equation}
where $g(p)$ is formulated as $g(p)=\max\{\frac{p-4}{2p}, \frac{1-p}{p}\}$.
\end{thm}

Theorem~\ref{thm:bound} guarantees that when perturbations to model parameters are small enough, the gradient-based corruption can accurately estimate the indicator with a low error rate. In Eq.(\ref{eq:bound}), the numerator is the proposed maximum loss change indicator and the denominator is the loss change with the gradient-based corruption. As we can see, when  $\epsilon$, the $p$-norm of the corruption vector, approaches zero, the term $O(\cdot)$ will also approach zero such that the ratio becomes one, meaning the gradient-based corruption can cause the same order infinitesimal loss change of the maximum loss change indicator.

\subsection{Multi-step Adversarial Parameter Corruption}

In this section, we further propose to seek the optimal parameter corruption by a multi-step optimization process based on the Projected Gradient Descent (PGD) algorithm~\citep{Unified-min-max}.

The corruption is initialized to zero. In the multi-step adversarial parameter corruption, we usually train the corruption for one epoch and the number of steps is $K=\lceil\nicefrac{|\mathcal{D}|}{|\mathcal{B}|} \rceil$, where $|\mathcal{B}|$ denotes the batch size. In each step, we generate a new corruption $\vect{a}_k\gets \mathcal{G}(\vect{a}_{k-1}, \alpha, S)$ based on the corruption $\vect{a}_{k-1}$, which breaks down a single perturbation into consecutive perturbations:
\begin{equation}
\vect{a}_k=\mathcal{G}(\vect{a}_{k-1}, \alpha, S)=\Pi_S(\vect{a}_{k-1}+\vect{u}_k).
\label{equ:generating}
\end{equation}

We first find an update $\vect{u}_k$ on the corruption $\vect{a}_{k-1}$ that approximately maximizes the loss in this step while controlling the step size $\|\vect{u}_k\|_p$ as $\alpha$. Then we project it to the closest vector in the constraint set $S$ and $\Pi_S$ is the projection function.

To find an optimal update, we obtain the gradient on batch $\mathcal{B}_i\subset \mathcal{D}$:
\begin{align}
\vect{g}_{k-1} \gets \nabla_{\vect{w}+\vect{a}_{k-1}}\mathcal{L}(\vect{w}+\vect{a}_{k-1}, {\bm\theta};\mathcal{B}_i).
\end{align}
Suppose the update on the vector $\vect{a}_{k-1}$ is $\vect{u}_k$, using the first-order Taylor expansion, the loss after perturbation can be rewritten as:
\begin{equation}
\mathcal{L}(\vect{w}+\vect{a}_{k-1}+\vect{u}_k;\mathcal{B}_i)\approx\mathcal{L}(\vect{w}+\vect{a}_{k-1};\mathcal{B}_i)+\vect{u}_{k}^\text{T}\vect{g}_{k-1}.
\end{equation}

To approximately maximize the loss in a single step with the step size $\alpha$, we adopt the gradient-based corruption to seek the $k$-th update $\vect{u}_k$ as the vector maximizing $\vect{g}_{k-1}^\text{T}\vect{u}_k$:
\begin{align}
\vect{u}_k=\argmax_{\|\vect{u}\|_p=\alpha}\vect{g}_{k-1}^\text{T}\vect{u}=\alpha\left(\text{sgn}(\vect{g})\odot\frac{|\vect{g}|^\frac{1}{p-1}}{\||\vect{g}|^\frac{1}{p-1}\|_p}\right);
\end{align}
where the step size $\|\vect{u}_k\|_p$ is constrained to $\alpha$.

After updating the corruption $\vect{a}'=\vect{a}_{k-1}+\vect{u}_k$, the corruption $\vect{a}'$ is not necessarily in the constrain set $S$. Therefore, we project the updated corruption  $\vect{a}'$ into the constraint set $S$, where $S=\{\vect{a}:\|\vect{a}\|_p\le \epsilon \text{ and } \|\vect{a}\|_0\le n\}$. We define the projection function $\Pi_S(\vect{a}')$ as finding the closest\footnote{Here we choose the Euclidean distance because $L_p$-norm distance 
cannot guarantee the uniqueness of the projected point when $p=1$ or $p=+\infty$.} vector to $\vect{a}'$ in $S$: 
\begin{align}
\Pi_S(\vect{a}')=\argmin_{\vect{y}\in S}\|\vect{y}-\vect{a}'\|_2. \label{eq:def_projection}
\end{align}
Solving Eq.~(\ref{eq:def_projection}) is difficult for general $p$. Fortunately, for two common cases ($L_2$ and $L_{+\infty}$), we have the closed-form solutions:
\begin{align}
\Pi_S(\vect{a}')&=\min\{\|\vect{h}\|_2, \epsilon\}\frac{\vect{h}}{\|\vect{h}\|_2}\quad (p=2);\\ 
\Pi_S(\vect{a}')&=\text{clip}(\vect{h}, -\epsilon, \epsilon)\quad (p={+\infty});
\end{align}
where $\text{clip}(\vect{h}, -\epsilon, \epsilon)$ clips every dimension of $\vect{h}$ into $[-\epsilon, \epsilon]$. 

It is easy to verify $\|\Pi_S(\vect{a'})\|_p\le\|\vect{a'}\|_p$. Therefore,
\begin{align}
\|\vect{a}_K\|_p&\le \|\vect{a}_{K-1}+\vect{u}_K\|_p\le \|\vect{a}_{K-1}\|_p+\|\vect{u}_K\|_p \\
&\le \cdots \le \|\vect{a}_0\|_p+\sum\limits_{k=1}^K\|\vect{u}_k\|_p=\sum\limits_{k=1}^K\|\vect{u}_k\|_p.
\end{align}
We can see that the sum of $\|\vect{u}_k\|_p$ can control the $L_p$-norm of $\|\vect{a}_K\|_p$. Therefore, we choose $\|\vect{u}_k\|_p$ as the definition of the step size. To ensure that the boundary of $S$ can be reached, the hyper-parameters $\epsilon, \alpha, K$ should satisfy $K\alpha\ge \epsilon$. Generally, $\alpha$ should increase as $\epsilon$ increases.

\section{Adversarial Parameter Defense Algorithm}

This section introduces the motivation of our algorithm for adversarial parameter defense, which is then elaborated with multi-step risk estimation.

\subsection{Motivation of Adversarial Parameter Training.}

Standard adversarial training~\citep{Explaining_and_Harnessing_Adversarial_Examples,Towards_Evaluating_the_Robustness_of_Neural_Networks,YOPO,freeLB,Unified-min-max} with respect to adversarial examples searches for optimal parameters to minimize the risk of adversarial input perturbation:
\begin{equation}
\vect{w}=\argmin\limits_{\vect{w}}\mathop{\mathbb{E}}\left[\max_{\vect{\Delta x}\in \delta}\mathcal{L}(\vect{w}; (\vect{x}+\vect{\Delta x}, y))\right],
\label{eq:1}
\end{equation}
where $\vect{\Delta x}$ is the input perturbation, and $\delta$ denotes a constraint set for adversarial examples.

Specifically, the proposed adversarial parameter defense extends the adversarial training to take into account the risk generated by parameter corruptions. The aim is to defend against parameter corruptions. If we would like to enhance certain parameters, we can categorize the parameters to $(\vect{w}, {\bm\theta})$, where $\vect{w}\in\mathbb{R}^k$ denotes the vector of $k$ parameters allowed to be corrupted and ${\bm\theta}$ denotes the parameters not allowed to be corrupted, which is an empty vector if all parameters are allowed to be corrupted.

In contrast to the standard adversarial training considering adversarial examples and minimizing the risk of input perturbation, the proposed defense algorithm aims to minimize the risk on dataset $\mathcal{D}$ for adversarial parameter attack in the constraint set $S$:
\begin{equation}
(\vect{w},{\bm\theta})=\argmin\limits_{(\vect{w},{\bm\theta})}\mathop{\mathbb{E}}\left[\max_{\vect{a}\in S}\mathcal{L}((\vect{w}+\vect{a},{\bm\theta}); \mathcal{B})\right],
\label{eq:2}
\end{equation}

Comparing Eq.~(\ref{eq:1}) and Eq.~(\ref{eq:2}), the defense to adversarial input attack and the proposed defense to adversarial parameter corruptions are the exact counterparts in minimizing the risk of adversarial perturbation, which improve model robustness from a different perspective.

\begin{algorithm}[!t]
   \caption{Adversarial Parameter Defense Algorithm}
   \label{alg:defend}
\begin{algorithmic}[1]
    \REQUIRE Parameters $(\vect{w}\in\mathbb{R}^k, {\bm\theta})$; loss $\mathcal{L}$ and dataset $\mathcal{D}$; corruption steps $K$ and step size $\alpha$; optimizer $\mathcal{O}$; training iterations; batch size $|\mathcal{B}_i|$.
    \STATE Prepare batches $\{\mathcal{B}_i\}$ and initialize $\vect{w}$ and ${\bm\theta}$.
    \WHILE {Training}
    \STATE $\vect{a}_0\gets \vect{0}_k.$
    \STATE  Calculate the initial loss: $\mathcal{L}((\vect{w}+\vect{a}_{0}, {\bm\theta});\mathcal{B}_i)$.
    \STATE $K$ is treated as $0$ in the early stage of training.
    \FOR {k = 1 to $K$}
    \STATE Generate $\vect{a}_{k}\gets \mathcal{G}(\vect{a}_{k-1}, \alpha, S)$ as Eq.(\ref{equ:generating}).
    \STATE  Calculate the risk: $\mathcal{L}((\vect{w}+\vect{a}_{k}, {\bm\theta});\mathcal{B}_i)$.
    \ENDFOR
    \STATE Update $\vect{w}, {\bm\theta}$ as minimizing Eq. (\ref{eq:target}).
    \ENDWHILE
\end{algorithmic}
\end{algorithm}

\subsection{Proposed Defense Algorithm}

The proposed defense algorithm adopts a min-max optimization process~\citep{Unified-min-max}, which first maximizes the loss change under parameter corruptions to estimate the risk of neural networks under parameter corruptions and then minimizes the estimated risk. The key point of defense is to know what to defend against. In our proposal, the risk that the defense is supposed to protect from is caused by parameter corruptions, and we propose a multi-step method based on virtual parameter corruptions to estimate such a risk.

The proposed defense considers the risk evaluated by averaging the risks of multiple parameter corruptions, which has the ability to defend against parameter corruption of various strengths. The parameter corruptions for estimating the risk are generated as follows: First, we generate multiple virtual adversarial parameter corruptions iteratively in $K$ steps and in every step, we generate a new virtual corruption $\vect{a}_k$ based on $\vect{a}_{k-1}$ with the multi-step adversarial parameter corruption algorithm: $\vect{a}_{k}\gets \mathcal{G}(\vect{a}_{k-1}, \alpha, S)$. Then suppose $K$ generated adversarial parameter attacks are $\vect{a}_1, \vect{a}_2,\cdots, \vect{a}_K$, we use the average loss on $K+1$ steps to estimate the risk for multiple virtual adversarial parameter corruptions. The target of the proposed algorithm is:
\begin{equation}
(\vect{w},{\bm\theta})=\argmin\limits_{(\vect{w},{\bm\theta})}\mathbb{E}\left[\sum\limits_{k=0}^K\frac{\mathcal{L}((\vect{w}+\vect{a}_k,{\bm\theta}); \mathcal{B})}{K+1}\right].
\label{eq:target}
\end{equation}

The algorithm is shown in Algorithm~\ref{alg:defend}. The constraint set $S$ defines the exploration space of parameter corruptions. We choose two common $L_p$-norms: $L_2$ or $L_{+\infty}$ and $n=k$. The step size is usually set as $\alpha=1.5\times\nicefrac{\epsilon}{K}$ to ensure that $K\alpha\ge\epsilon$. It should be noted that, in the early stage of training, the defense algorithm may harm the learning. Thus, we set $K$ as $0$ in the early stage, i.e., ordinary training process without defense. We also choose a start epoch and start to adopt the defense algorithm at the start epoch. We also discuss the computation complexity of the algorithm in Appendix B.

\subsection{Theoretical Analysis}
\label{sec:theoretical}

A direct method is to minimize the estimation of the parameter corruption risk for only a single parameter corruption $\mathbb{E}_{\mathcal{B}\subset\mathcal{D}}\left[\mathcal{L}(\vect{w}+\vect{\hat a},{\bm\theta}; \mathcal{B})\right]$, where $\vect{\hat a}$ denotes an optimal parameter corruption to estimate the risk. Intuitively, the method estimating the risk with multiple corruptions could be a method with lower generalization error compared to the direct method, since the method estimating the risk with multiple steps considers multiple parameter corruptions while the direct method only considers one corruption. 

Based on previous work on PAC-Bayes bound~\citep{Bayes_bound} and inspired by \cite{Sharpness-Aware_Minimization,Adversarial_data_Weight_Perturbation}, we provide a theoretical analysis in Theorem~\ref{thm:generalization_error} that the generalization error bound relates to the proposed average loss change and maximum loss change indicators under a general $L_p$-norm constraint.

\begin{thm}[Relation between proposed indicators and generalization error bound]
\label{thm:generalization_error}
Assume the prior over the parameters $\vect{w}$ is $N(\vect{0}, \sigma^2 \mathbf{I})$. Given the constraint set $S=\{\vect{a}\in \mathbb{R}^k:\|\vect{a}\|_2= \epsilon\}$ and we choose the expectation error rate as the loss function, with probability 1-$\delta$ over the choice of the training set $\mathcal{D}\sim \mathcal{D}_1$, when $\mathcal{L}(\vect{w}, \mathcal{D}_1)$ is convex in the neighborhood of $\vect{w}$, \footnote{Note that $\mathcal{L}$ is only required to be convex in the neighbourhood of $\vect{w}$ instead of the entire $\mathbb{R}^k$.} the following generalization error bound holds,
\begin{align}
\mathcal{L}(\vect{w}, \mathcal{D}_1)\le
\mathcal{L}(\vect{w}, \mathcal{D})+
\Delta_\text{ave} \mathcal{L}(\vect{w}, S, \mathcal{D})+\mathcal{R},
\end{align}
where $R=\sqrt{\frac{C+\log\frac{|\mathcal{D}|}{\delta}}{2(|\mathcal{D}|-1)}}+o(\epsilon^2), C=\frac{\epsilon^2+\|\vect{w}\|^2_2}{2\sigma^2}-\frac{k}{2}+\frac{k}{2}\log\frac{k\sigma^2}{\epsilon^2}$ is not determined by $\mathcal{|D|}$ and $\delta$.

Generally, when $S_1=\{\vect{a}\in\mathbb{R}^k:\|\vect{a}\|_p\le\epsilon\}$, we have,
\begin{align}
\mathcal{L}(\vect{w}, \mathcal{D}_1)\le
\mathcal{L}(\vect{w}, \mathcal{D})+
\Delta_\text{max} \mathcal{L}(\vect{w}, S_1, \mathcal{D})+\mathcal{R}_1,
\end{align}
where $R_1=\sqrt{\frac{C_1+\log\frac{|\mathcal{D}|}{\delta}}{2(|\mathcal{D}|-1)}}+o(\epsilon^2),C_1=\frac{\epsilon^2+\beta_p^2\|\vect{w}\|^2_2}{2\beta_p^2\sigma^2}-\frac{k}{2}+\frac{k}{2}\log\frac{k\sigma^2\beta_p^2}{\epsilon^2}$ is not determined by $\mathcal{|D|}$ and $\delta$, here $\beta_p=\max\{1, k^{1/p-1/2}\}$.
\end{thm}

Thus, the generalization error can be bounded by the proposed indicators, which may explain why the proposed algorithm for adversarial parameter defense can improve the accuracy of the neural networks.

\begin{table*}[ht]
\caption{Comparisons of baseline and models with defense algorithm under multi-step risk estimation (the number of corrupted parameters is not limited, unless otherwise stated), with further study on other corruption algorithms.}
\label{tab:main-results}
\tiny
\setlength{\tabcolsep}{1pt}
\centering
\begin{tabular}{@{}lcccccccccccc@{}}
\toprule
 
  \bf Datasest & \multicolumn{3}{c}{\textbf{CIFAR-10 (Acc)} } &  \multicolumn{3}{c}{\textbf{VOC (mAP)}} &  \multicolumn{3}{c}{\textbf{En-Vi (BLEU)}} & \multicolumn{3}{c}{\textbf{De-En (BLEU)}}\\ 
\cmidrule(r){1-1} \cmidrule(r){2-4} \cmidrule(r){5-7} \cmidrule(r){8-10} \cmidrule(r){11-13}
   Corruption Approach & Settings& ResNet-101  & w/ Defense & Settings& ResNet-101 & w/ Defense & Settings& Transformer & w/ Defense & Settings& Transformer &  w/ Defense \\

\midrule
 w/o Corruption & - & 94.5 & \textbf{96.3 (+1.8)} & - & 74.9 & \textbf{75.8 (+0.9)} & - &  30.64 & \textbf{31.09 (+0.45)}  & - &  35.32 & \textbf{35.88 (+0.56)}\\
\midrule

 \multirow{3}{1 in}{Multi-step Corruption with different $\epsilon$. ($L_2$).} & 0.02 & 94.0 & 96.0 & 0.05 & 74.2 & 75.3 & 0.20 & 19.85 & 27.72 & 0.20 & 32.66 & 35.65 \\
 & 0.05 & 89.8 & 95.1 & 0.10 & 73.1 & 73.2 & 0.22 & 16.78 & 26.21 & 0.30 & 19.13 & 33.71 \\
 & 0.10 & 55.1 & 55.3 & 0.20 & 61.9 & 62.9 & 0.24 & 11.78 & 23.65 & 0.35 & 15.07 & 22.87 \\
 \cmidrule(r){1-1} \cmidrule(r){2-4} \cmidrule(r){5-7} \cmidrule(r){8-10} \cmidrule(r){11-13}
 
  \multirow{3}{1 in}{Multi-step Corruption with different $\epsilon$. ($L_{+\infty}$).} & $1\times10^{-5}$ & 94.3 & 95.9 & $5\times10^{-5}$ & 72.3 & 74.8 & $2\times10^{-4}$ & 30.64 & 31.09 & $2\times10^{-4}$ & 35.12 & 35.86 \\
 & $2\times10^{-5}$ & 92.7 & 95.8 & $1\times10^{-4}$ & 43.6 & 64.8 & $4\times10^{-4}$ & 29.84 & 30.48 & $4\times10^{-4}$ & 30.79 & 35.41 \\
 & $5\times10^{-5}$ & 62.6 & 67.3 & $2\times10^{-4}$ & \phantom{0}0.0 & \phantom{0}3.5 & $5\times10^{-4}$ & \phantom{0}6.66 & 28.59 & $5\times10^{-4}$ & 24.89 & 31.61 \\

 \cmidrule(r){1-1} \cmidrule(r){2-4} \cmidrule(r){5-7} \cmidrule(r){8-10} \cmidrule(r){11-13}
 
   \multirow{3}{1 in}{Multi-step Corruption with different $\epsilon$. ($n=100, L_{+\infty}$).}
   & $2\times10^{-4}$ & 94.5 & 96.1 & $7\times10^{-4}$ & 74.7 & 75.8 & $1\times10^{-3}$ & 30.59 & 31.09 & $2\times10^{-3}$ & 35.26 & 35.86 \\
 & $5\times10^{-4}$ & 12.8 & 21.7 & $8\times10^{-4}$ & \phantom{0}0.0 & 66.8 & $2\times10^{-3}$ & 16.54 & 30.71 & $5\times10^{-3}$ & 18.30 & 35.72 \\
 & $1\times10^{-3}$ & 10.0 & 10.0 & $1\times10^{-3}$ & \phantom{0}0.0 & \phantom{0}0.0 & $5\times10^{-3}$ & \phantom{0}0.00 & \phantom{0}0.00 & $1\times10^{-2}$ & \phantom{0}1.64 & 28.14 \\

\midrule
 
   \multirow{3}{1 in}{Gradient-based Corruption with different $\epsilon$. ($L_2$).} & 0.1 & 82.8 & 94.5 & 0.2 & 68.4 & 69.4 & 0.5 & 26.04 & 30.17 & 0.2 & 34.61 & 35.64 \\
 & 0.2 & 49.7 & 88.2 & 0.5 & 27.4 & 36.3 & 1.0 & \phantom{0}2.57 & 14.89 & 0.5 & 31.85 & 34.88 \\
 & 0.5 & 19.7 & 41.7 & 1.0 & \phantom{0}4.5 & \phantom{0}5.5 & 2.0 & \phantom{0}0.00 & \phantom{0}1.90 & 1.0 & 12.31 & 31.37 \\
 \cmidrule(r){1-1} \cmidrule(r){2-4} \cmidrule(r){5-7} \cmidrule(r){8-10} \cmidrule(r){11-13}
 
   \multirow{3}{1 in}{Gradient-based Corruption with different $\epsilon$. ($L_{+\infty}$).} & $5\times10^{-5}$ & 81.0 & 93.4 & $2\times10^{-5}$ & 74.4 & 74.9 & $5\times10^{-4}$ & 30.03 & 30.68 & $1\times10^{-3}$ & 32.01 & 34.22 \\
 & $1\times10^{-4}$ & 43.2 & 81.7 & $5\times10^{-5}$ & 70.8 & 71.9 & $1\times10^{-3}$ & 27.78 & 29.71 & $2\times10^{-3}$ & 24.66 & 30.54 \\
 & $2\times10^{-4}$ & 24.5 & 28.3 & $1\times10^{-4}$ & 52.0 & 58.7 &  $2\times10^{-3}$ & 18.84 & 26.16  & $5\times10^{-3}$ & \phantom{0}2.28 & 13.80 \\

\midrule
 
   \multirow{3}{1 in}{Gaussian noise $N(0, \sigma^2)$ on parameters with different $\sigma$.} & $2\times10^{-3}$ & 94.1 & 96.0 & $1\times10^{-3}$ & 68.0 & 75.2 & $5\times10^{-3}$ & 29.46 & 30.77  & $1\times 10^{-2}$ & 34.71 & 35.57 \\
 & $5\times10^{-3}$ & 92.5 & 94.6 & $2\times10^{-3}$ & \phantom{0}9.9 & 74.9 & $1\times10^{-2}$ & 29.06 & 30.25 & $2\times 10^{-2}$ & 32.09 & 33.97 \\
 & $1\times10^{-2}$ & 21.3 & 29.1 & $5\times10^{-3}$ & \phantom{0}0.0 & \phantom{0}0.4 & $2\times10^{-2}$ & 24.85 & 27.21  &  $3\times 10^{-2}$ & 30.29 &  33.36\\
 \cmidrule(r){1-1} \cmidrule(r){2-4} \cmidrule(r){5-7} \cmidrule(r){8-10} \cmidrule(r){11-13}
 
    \multirow{3}{1 in}{Uniform noise $U(-b, b)$  on parameters with different settings of $b$.} & $2\times10^{-3}$ & 94.3 & 96.0 & $2\times10^{-3}$ & 69.8 & 75.6 & $1\times 10^{-2}$  & 29.33 & 30.98 & $2\times 10^{-2}$ & 34.33 & 35.19\\
 & $5\times10^{-3}$ & 92.4 & 95.7 & $5\times10^{-3}$ & 21.7 & 71.4 & $2\times 10^{-2}$  & 28.51 & 30.27 & $5\times 10^{-2}$ & 17.06 & 25.39 \\
 & $1\times10^{-2}$ & 17.1 & 93.0 & $1\times10^{-2}$ & \phantom{0}0.0 & \phantom{0}0.2 & $5\times 10^{-2}$ & 17.06 & 25.39  &  $8\times 10^{-2}$ & 13.10 & 21.29 \\

\midrule
 
  \multirow{2}{1 in}{Tensor-RT~\citep{tensorRT} weight quantization.}
  & 5 bit & 94.5 & 96.2 & 6 bit & 73.0 & 74.8 & 6 bit & 30.31 & 31.02 & 7 bit & 34.79 & 35.74 \\
  & 4 bit & 84.0 & 92.2 & 5 bit & 55.5 & 65.8 & 5 bit & 28.62 & 29.56 & 6 bit & 34.12 & 35.13 \\
\bottomrule
\end{tabular}
\end{table*}

\section{Experiments}
In this section, we evaluate the proposed defense algorithm over two popular deep neural networks (ResNet~\citep{resnet} and Transformer~\citep{transformer}) across both CV and NLP benchmark datasets. We first describe experimental settings. Then we summarize the main results in Table~\ref{tab:main-results} and discuss the experimental results of the proposed defense algorithm. 

\subsection{Experimental Settings}

We conduct experiments using \textbf{ResNet-101} on two classic computer vision datasets, i.e., CIFAR-10 image classification dataset (\textbf{CIFAR-10}) ~\citep{CIFAR-10} and the PASCAL Visual Object Classes Challenge: PASCAL VOC 2007 dataset (\textbf{VOC})~\citep{VOC}. The evaluation metrics are accuracy (\textbf{Acc}) and mean average precision (\textbf{mAP}), respectively. For \textbf{Transformer}, we use IWSLT 15 English-Vietnamese (\textbf{En-Vi})~\citep{2015iwslt} and IWSLT 14 German-English (\textbf{De-En})~\citep{2014iwslt} with the evaluation metric of \textbf{BLEU} score. Experiments of performance without corruption are repeated $3$ times for hypothesis testings.

To verify the robustness of models, what are baselines and models with defense are tested by multiple corruption approaches, including: (1) Our proposed multi-step adversarial corruption method; (2) Our proposed gradient-based corruption method; (3) Random Gaussian or uniform noises on parameters to simulate random corruptions; (4) Tensor-RT~\citep{tensorRT} weight quantization method, which quantifies parameters into $n$-bit signed integers. Please refer to Appendix A for details of experimental settings and parameter corruption approaches.

\subsection{Results}
In this section, we report the main results of our proposed defense algorithm. The comparative results between baseline models and models with the proposed defense algorithm are shown in Table~\ref{tab:main-results}. First, the proposed defense algorithm achieves better overall accuracy, meaning that enhancing the robustness of parameters can achieve  better generalization ability. As analyzed in Section~\ref{sec:theoretical}, a flat minimum with better parameter robustness tends to imply better generalization. Second, the models enhanced by the proposed defense methods demonstrate more resistance under multiple parameter corruption approaches, including adversarial parameter corruptions, random noises, or real-world quantization.

\section{Further Analysis}

In this section, we first study the influence of the experimental settings and compare our methods with some variants. Then, we further verify the effectiveness of the proposed defense method on BERT~\citep{Bert}. Last, we probe into and visualize the vulnerability of the models at different layers. In general, we find that not all layers of models are sufficiently enhanced by the proposed defense algorithm. To tackle this issue, we select certain layers of the model to defend, instead of the entire model.

\begin{table}[!t]
\caption{Results of the influence of the experimental settings on the De-En dataset. The number of corrupted parameters in defense is not limited and $L_{+\infty}$ constraint are adopted unless otherwise stated.}
\label{tab:ablation}
\scriptsize
\setlength{\tabcolsep}{5pt}
\centering
\begin{tabular}{@{}lccc@{}}
\toprule
\textbf{Settings} & $K$ & $\epsilon$ & \textbf{BLEU} \\
\midrule
w/o Defense & - & - & 35.32 \\
w/ Defense & 2 & 0.0006 & 35.88 \\
\midrule
\multirow{3}{1.5in}{w/ different $K$} & 1 & 0.0006 & 35.69 \\
 & 2 & 0.0006 & 35.88 \\
 & 3 & 0.0006 & 35.81 \\
\midrule
\multirow{3}{1.5in}{w/ different $K$, \\random initialization} & 1 & 0.0006 & 35.59\\
 & 2 & 0.0006 & 35.73 \\
 & 3 & 0.0006 & 35.65 \\
\midrule
\multirow{3}{1.5in}{w/ different $\epsilon$, $K=1$} & 1 & 0.0004 & 35.68 \\
 & 1 & 0.0006 & 35.83 \\
 & 1 & 0.0008 & 35.74 \\
\midrule
\multirow{3}{1.5in}{w/ different $\epsilon$, $K=2$} & 2 & 0.0004 & 35.72 \\
 & 2 & 0.0006 & 35.88 \\
 & 2 & 0.0008 & 35.86 \\
\midrule
Best settings with $L_2$ constraint. & 2 & 0.4 & 35.41 \\
\bottomrule
\end{tabular}
\end{table}

\subsection{Influence of the Experimental Settings}

We study the influence of the experimental settings by adopting the De-En dataset. Results are in Table~\ref{tab:ablation}.

\textbf{Influence of $K$.} With the steps $K$ increasing in the defense method, the performance increases first and then drops, demonstrating an optimized configuration is $K=2$.

\textbf{Influence of the initial corruption.} In the defense method, we adopt zero initialization, namely the initial corruption is $\vect{a}_0=\vect{0}$. Inspired by the FreeLB~\citep{freeLB} algorithm in adversarial training with respect to adversarial examples, we also try to initialize $\vect{a}_0$ randomly and control $\|\vect{a}_0\|_p=\epsilon$. But experimental results show that the defense method with zero initialization method outperforms defense with random initialization under multiple $K$.

\textbf{Influence of $\epsilon$.}  With $\epsilon$ increasing in the defense method, the performance increases first and then drops under multiple $K$. This is probably because that when the magnitude of the parameter corruption is too large, the model will collapse after corruption, which eventually harms the learning. The best configuration of $\epsilon$ is $\epsilon=0.0006$ on the De-En dataset.

\textbf{Choice of the $L_p$ constraint.}  On the De-En dataset, we also consider the $L_2$ constraint. We grid search $K, \epsilon$ and the best settings with $L_2$ constraint is $K=2, \epsilon=0.4$. The BLEU is $35.41$ and is lower than $35.88$ with the $L_{+\infty}$ constraint. However, the choice of the $L_p$ constraint depends on the task. For the ResNet-101 model on the CIFAR-10 and VOC datasets, the $L_2$ constraint is a better choice. While for the Transformer model on the De-En and En-Vi datasets, the $L_{+\infty}$ constraint is a better choice.

\textbf{Choice of hyperparameters.} To conclude, the hyperparameters are mainly task-dependent and we can search the best configurations by grid-search except the cases of the zero initialization. In particular, the performance can be treated as a unimodal function of $\epsilon$ or $K$ approximately, whose best configuration can be determined easily in hyperparameter search.

\subsection{Comparison with Variants}

Some previous researches in the field of adversarial training with respect to the parameters, including our previous work ACRT~\citep{attack-paper}, SAM~\citep{Sharpness-Aware_Minimization}, and AWP~\citep{Adversarial_data_Weight_Perturbation}, can be seen as variants of our proposed defense algorithm.

\textbf{ACRT~\citep{attack-paper}: Adversarial Corruption-Resistant Training and SAM~\citep{Sharpness-Aware_Minimization}: Sharpness Aware Minimization.} In our previous work, we propose ACRT~\citep{attack-paper} (Adversarial Corruption-Resistant Training) to improve the resistant to parameter corruptions of DNN. It considers the risk after a gradient-based corruption, namely:
\begin{align}
\vect{w}=\argmin\limits_{\vect{w}}\big[(1-\alpha)\mathcal{L}(\vect{w}; \mathcal{D})+\alpha\mathcal{L}(\vect{w}+\vect{\hat a}; \mathcal{D})\big],
\label{eq:ACRT}
\end{align}
where $\vect{\hat a}$ is a gradient-based corruption on the corruption constraint. $\bm\theta$ is omitted since all parameters are allowed to be corrupted here. We can also conduct a Taylor expansion on the loss and minimize the substitutive loss instead:
\begin{equation}
\begin{aligned}
& (1-\alpha)\mathcal{L}(\vect{w}; \mathcal{D})+\alpha\mathcal{L}(\vect{w}+\vect{\hat a}; \mathcal{D})\\
& \approx \mathcal{L}(\vect{w}; \mathcal{D}) + \alpha\vect{\hat a}^\text{T}\nabla_\vect{w}\mathcal{L}(\vect{w}; \mathcal{D}).
\label{eq:ACRT-gradient}
\end{aligned}
\end{equation}
SAM~\citep{Sharpness-Aware_Minimization} (Sharpness Aware Minimization) is similar to the ACRT method. SAM proposes to minimize the sharpness risk, which is defined as $\mathcal{L}(\vect{w}+\vect{\hat a}; \mathcal{D})$. SAM adopts an $L_2$ constraint as the corruption constraint and can be treated as an ACRT method with $\alpha=1$. In our implementation, we also consider the $L_{+\infty}$ constraint and the setting of the start epoch. $\alpha$ is set to $1$ following SAM~\citep{Sharpness-Aware_Minimization}.

\begin{table}[!t]
\caption{\cc{Results of models with different defense methods under multi-step corruptions} ($L_
{+\infty}$) and Tensor-RT~\citep{tensorRT} quantization.}
\label{tab:variants_defense}
\tiny
\setlength{\tabcolsep}{6pt}
\centering
\begin{tabular}{@{}lrrrr@{}}
\toprule
\multicolumn{5}{c}{\textbf{CIFAR-10 (Acc)}} \\
\midrule
\bf Corruption Approach & Baseline &  ACRT~\citep{attack-paper} &  AWP~\citep{Adversarial_data_Weight_Perturbation} & Proposed \\ 
\midrule
 w/o Corruption & 94.5 & 96.2 & 96.1 & 96.3 \\
 \midrule
 Multi-step ($\epsilon=1\times 10^{-5}$) & 94.3 & 96.0 & 95.7 & 95.9 \\
 Multi-step ($\epsilon=2\times 10^{-5}$) & 92.7 & 95.8 & 95.5 & 95.8 \\
Quantization (5 bit) & 94.5 & 95.5 & 95.1 & 96.2 \\
Quantization (4 bit) & 84.0 & 89.5 & 88.1 & 92.2 \\
\midrule\multicolumn{5}{c}{\textbf{VOC (mAP)}}\\
\midrule
 \bf Corruption Approach & Baseline &  ACRT~\citep{attack-paper} &  AWP~\citep{Adversarial_data_Weight_Perturbation} & Proposed \\ 
\midrule
 w/o Corruption & 74.9 & 75.1 & 75.0 & 75.8 \\
 \midrule
 Multi-step ($\epsilon=5\times 10^{-5}$) & 72.3 & 74.9 & 74.2 & 74.8 \\
 Multi-step ($\epsilon=1\times 10^{-4}$) & 43.6 & 63.8 & 67.0 & 64.8 \\
 +Quantization (6 bit) & 73.0 & 73.5 & 73.4 & 74.8 \\
 +Quantization (5 bit) & 55.5 & 54.5 & 58.8 & 65.8\\
\midrule
\multicolumn{5}{c}{\textbf{En-Vi (BLEU)}}\\
\midrule
 \bf Corruption Approach & Baseline &  ACRT~\citep{attack-paper} &  AWP~\citep{Adversarial_data_Weight_Perturbation} & Proposed \\ 
\midrule
 w/o Corruption & 30.64 & 30.71 & 30.62 & 31.09 \\
 \midrule
 Multi-step ($\epsilon=4\times 10^{-4}$) & 29.84 & 30.20 & 30.38 &30.48 \\
 Multi-step ($\epsilon=5\times 10^{-4}$)  & 6.66 & 27.93 & 28.68 & 28.59 \\
 +Quantization (7 bit) & 30.55 & 30.58 & 30.62 & 37.09\\
 +Quantization (6 bit) & 30.31 & 30.43 & 30.35 & 31.02\\
\midrule
\multicolumn{5}{c}{\textbf{De-En (BLEU)}}\\
\midrule
 \bf Corruption Approach & Baseline &  ACRT~\citep{attack-paper} &  AWP~\citep{Adversarial_data_Weight_Perturbation} & Proposed \\ 
\midrule
 w/o Corruption & 35.32 & 35.53 & 35.49 & 35.88\\
 \midrule
 Multi-step ($\epsilon=4\times 10^{-4}$) & 30.79 & 34.48 & 34.39 & 35.41\\
 Multi-step ($\epsilon=5\times 10^{-4}$) & 24.89 &30.10 & 30.04 & 31.63 \\
 +Quantization (7 bit) & 34.79 &35.22 & 35.20 & 35.74\\
 +Quantization (6 bit) & 34.12 & 35.07 & 34.93 & 35.13\\
\bottomrule
\end{tabular}
\end{table}

\textbf{AWP~\citep{Adversarial_data_Weight_Perturbation}: Adversarial Weight Perturbation.} AWP~\citep{Adversarial_data_Weight_Perturbation} considers the risk after a virtual parameter corruption under adversarial examples, namely:
\begin{align}
\vect{w}=\argmin\limits_{\vect{w}}\mathcal{L}(\vect{w}+\vect{a}; \mathcal{D}'),
\label{eq:WAP}
\end{align}
where $\vect{a}$ is a virtual parameter corruption solved by PGD and $\mathcal{D}'$ includes adversarial data instances. $\bm\theta$ is omitted since all parameters are allowed to be corrupted here. We also consider the setting of the start epoch in experiments.


\cc{We conduct experiments to compare the proposed defense method and its variants. The details of implementation are in Appendix A. Results of the proposed defense method and its variants under multiple parameter corruptions are shown in Table}~\ref{tab:variants_defense}, and results of the generalization ability are shown in Table~\ref{tab:variants}. As shown in Table~\ref{tab:variants_defense}, both our proposed defense method and its variants can improve the robustness of neural networks under adversarial parameter corruptions or weight quantization, which leads to improvements in the generalization ability. Furthermore, as shown in Table~\ref{tab:variants}, our proposed defense method has better generalization ability than its variants. It is maybe because previous researches only consider a single risk after the corruption, while our proposed defense method averages the risks for multiple corruptions, thus has lower generalization error, intuitively.


\begin{table}[!t]
\caption{Results of the proposed defense method and its variants.}
\label{tab:variants}
\scriptsize
\setlength{\tabcolsep}{2pt}
\centering
\begin{tabular}{@{}lcccc@{}}
\toprule
 \bf Datasest & \textbf{CIFAR-10} &  \textbf{VOC} &  \textbf{En-Vi} & \textbf{De-En}\\ 
 \midrule
 Baseline & 94.46$\pm$0.164 & 74.90$\pm$0.200 & 30.64$\pm$0.015 & 35.32$\pm$0.131 \\
 \midrule
 ACRT~\citep{attack-paper} & 96.23$\pm$0.031 & 75.13$\pm$0.306 & 30.71$\pm$0.053 & 35.53$\pm$0.146 \\
 AWP~\citep{Adversarial_data_Weight_Perturbation} & 96.08$\pm$0.093 & 75.03$\pm$0.153 & 30.62$\pm$0.167 & 35.49$\pm$0.229 \\
 \midrule
 Proposed & \textbf{96.34$\pm$0.076} & \textbf{75.77$\pm$0.152} & \textbf{31.09$\pm$0.102} & \textbf{35.88$\pm$0.053} \\
\bottomrule
\end{tabular}
\end{table}

\begin{table}[!t]
\caption{Results of enhancing BERT on SST-2.}
\label{tab:bert}
\small
\setlength{\tabcolsep}{1pt}
\centering
\begin{tabular}{@{}lcccc@{}}
\toprule
  \textbf{Approach} & Baseline & ACRT~\citep{attack-paper} & AWP~\citep{Adversarial_data_Weight_Perturbation} & Proposed \\
 \midrule
 \textbf{Acc} & 92.03$\pm$0.55 & \multicolumn{2}{c}{not converge} & \textbf{92.78$\pm$0.18}\\
\bottomrule
\end{tabular}
\end{table}

\begin{figure*}[!t]
\centering
\subcaptionbox{Transformer w/o defense. (En-Vi)}{\includegraphics[height=2.5 in,width=0.3\linewidth]{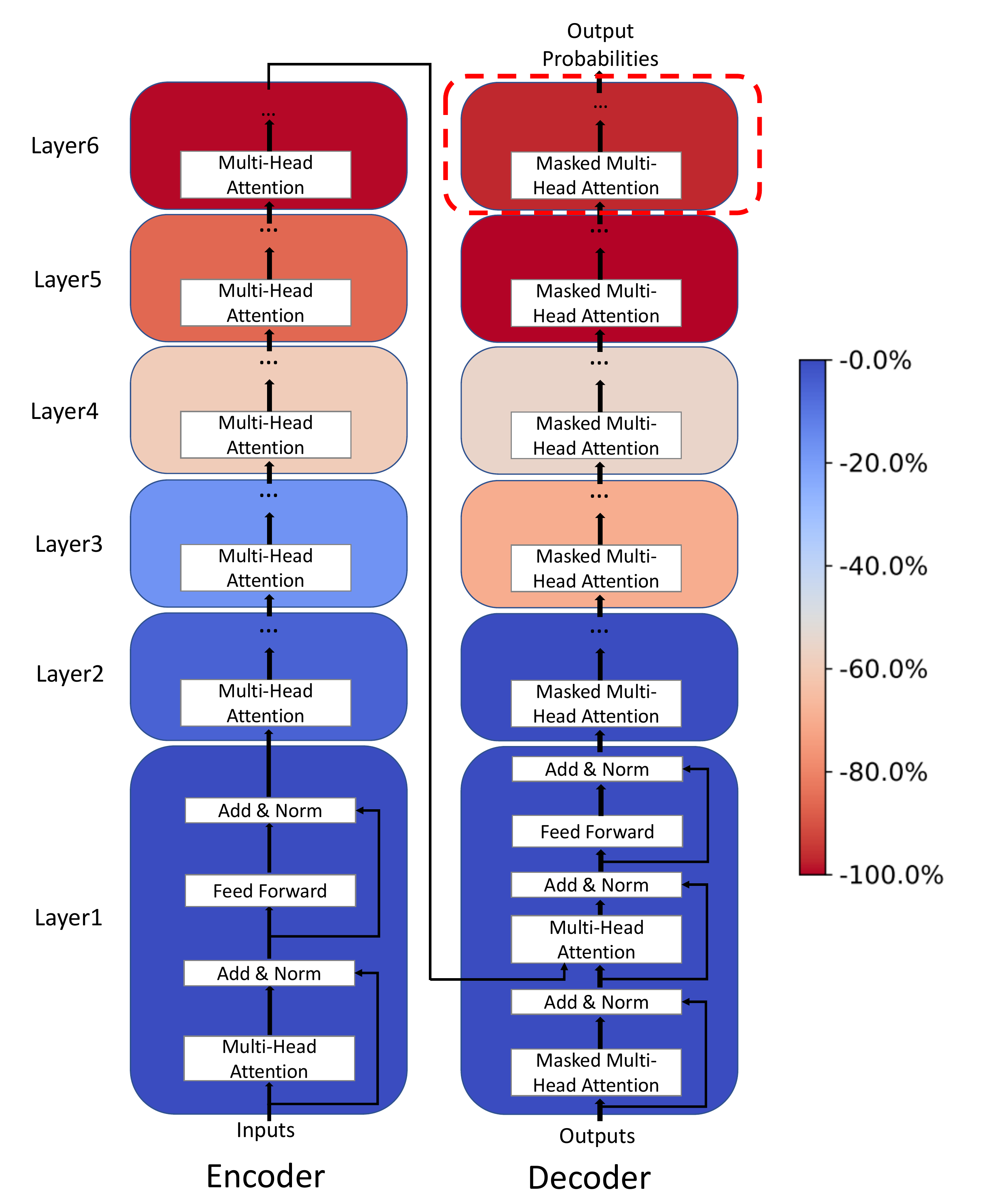}}
\hfil
\subcaptionbox{ Defend all layers. (En-Vi)}{\includegraphics[height=2.5 in,width=0.3\linewidth]{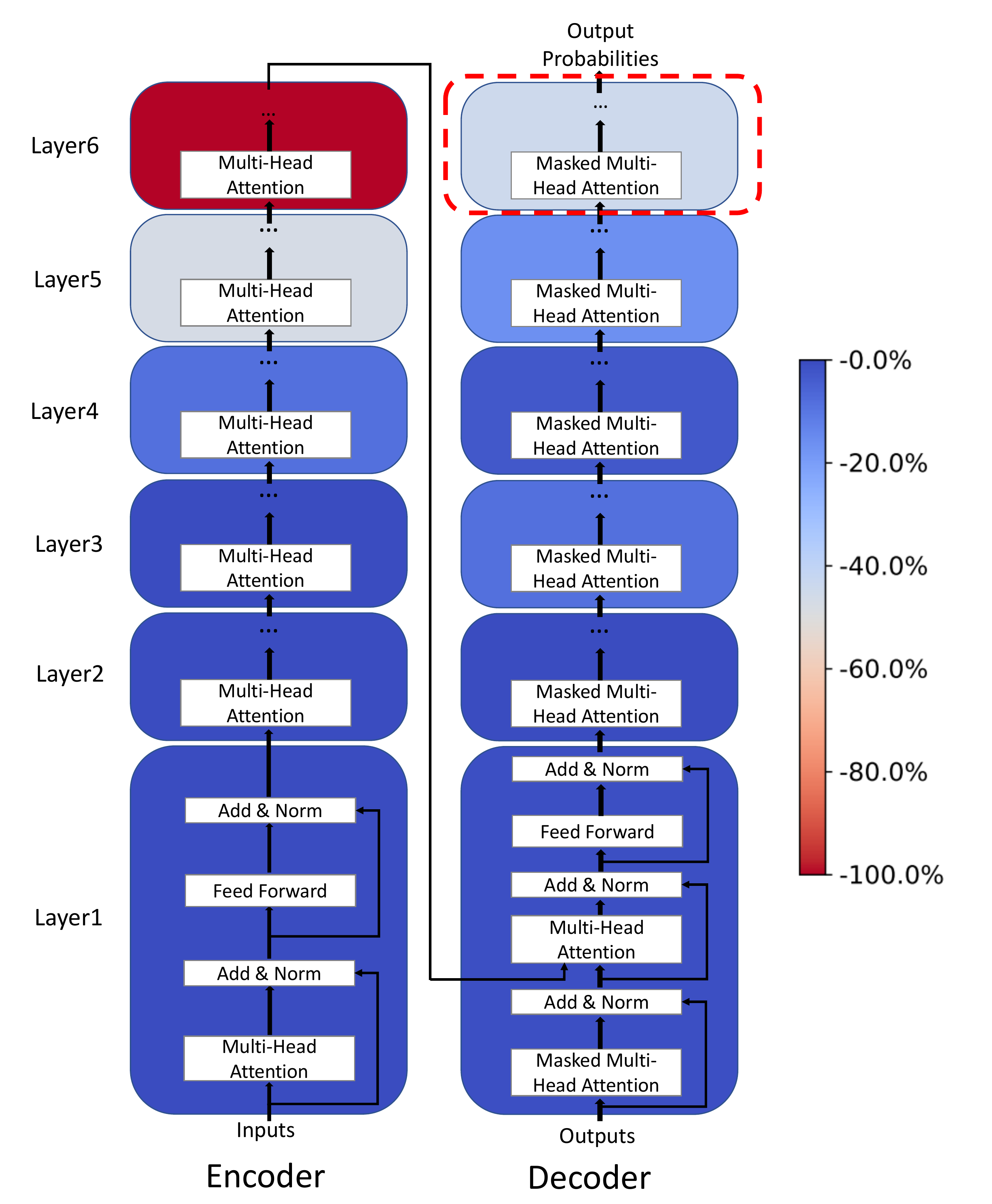}}
\hfil
\subcaptionbox{ Defend last layer. (En-Vi)}{\includegraphics[height=2.5 in,width=0.3\linewidth]{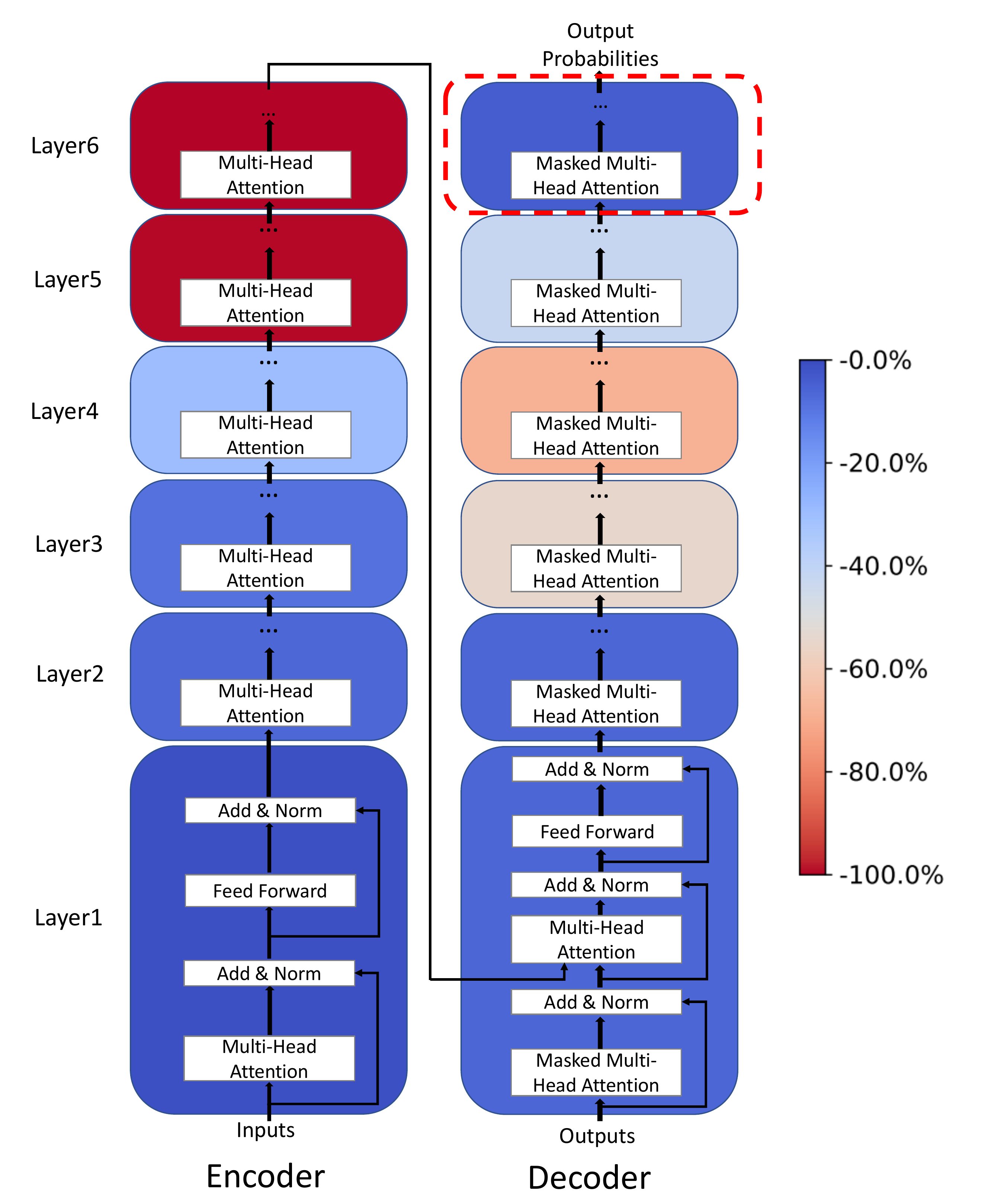}}
\hfil
\subcaptionbox{Transformer w/o defense. (De-En)}{\includegraphics[height=2.5 in,width=0.3\linewidth]{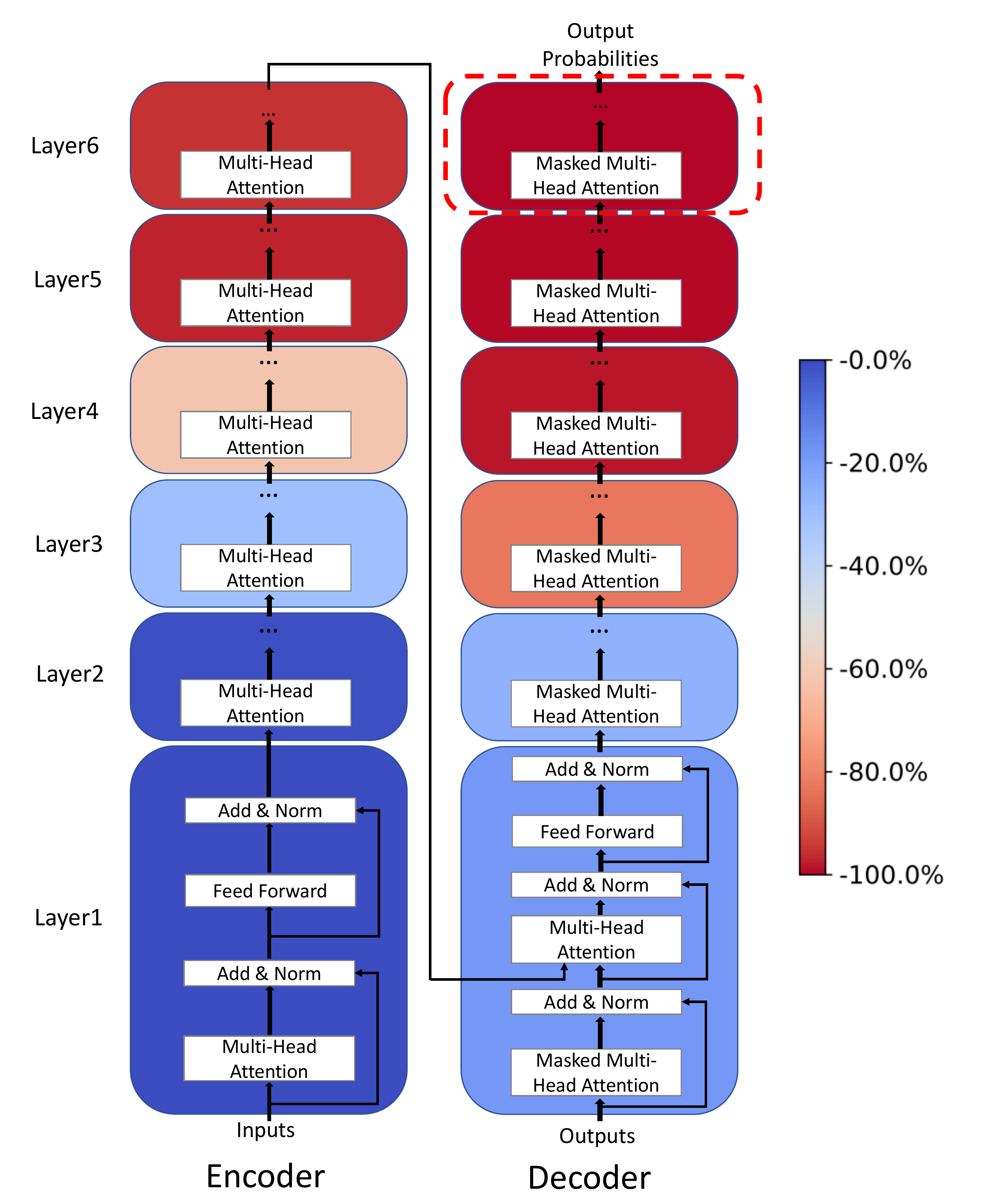}}
\hfil
\subcaptionbox{ Defend all layers. (De-En)}{\includegraphics[height=2.5 in,width=0.3\linewidth]{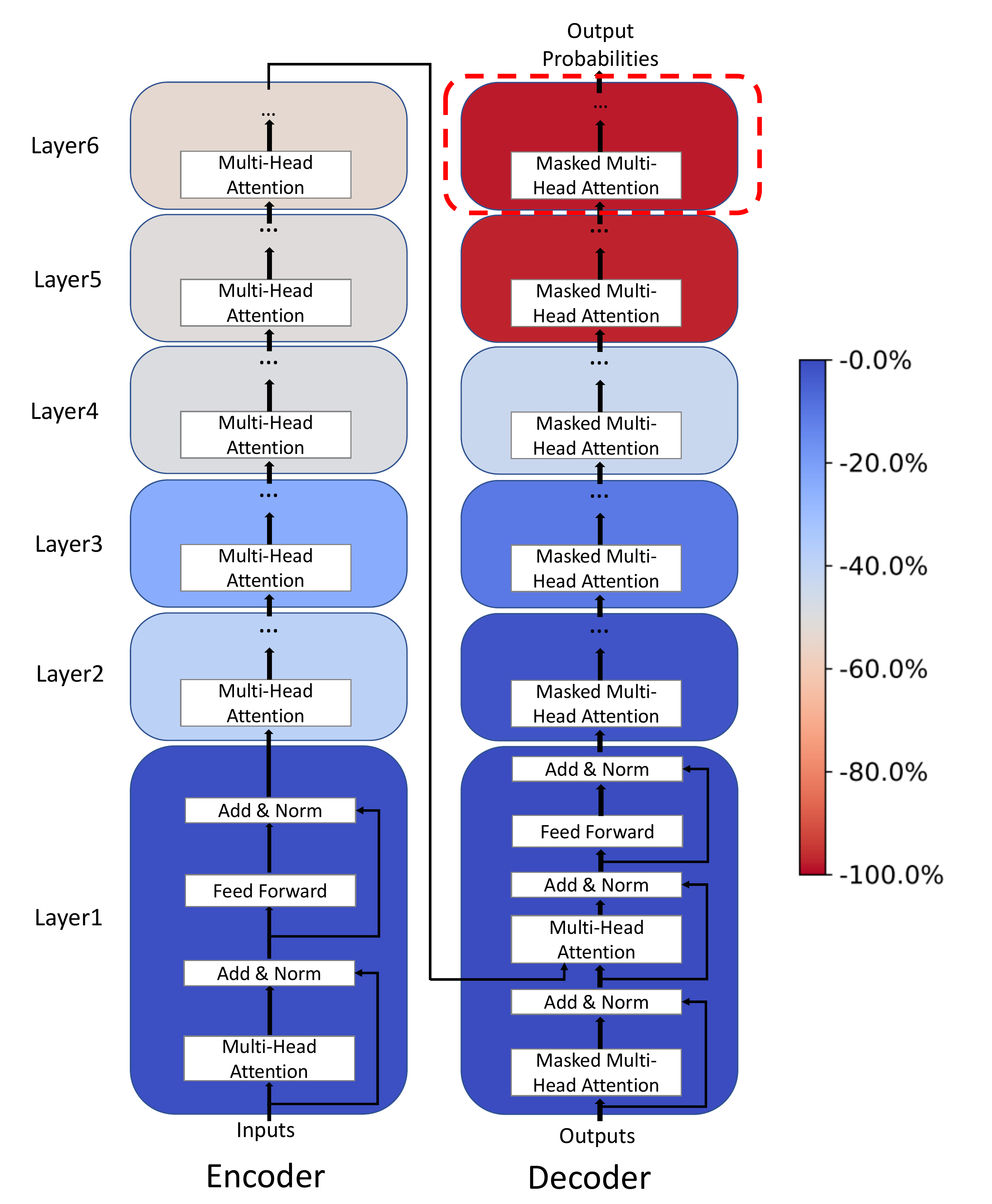}}
\hfil
\subcaptionbox{ Defend last layer. (De-En)}{\includegraphics[height=2.5 in, width=0.3\linewidth]{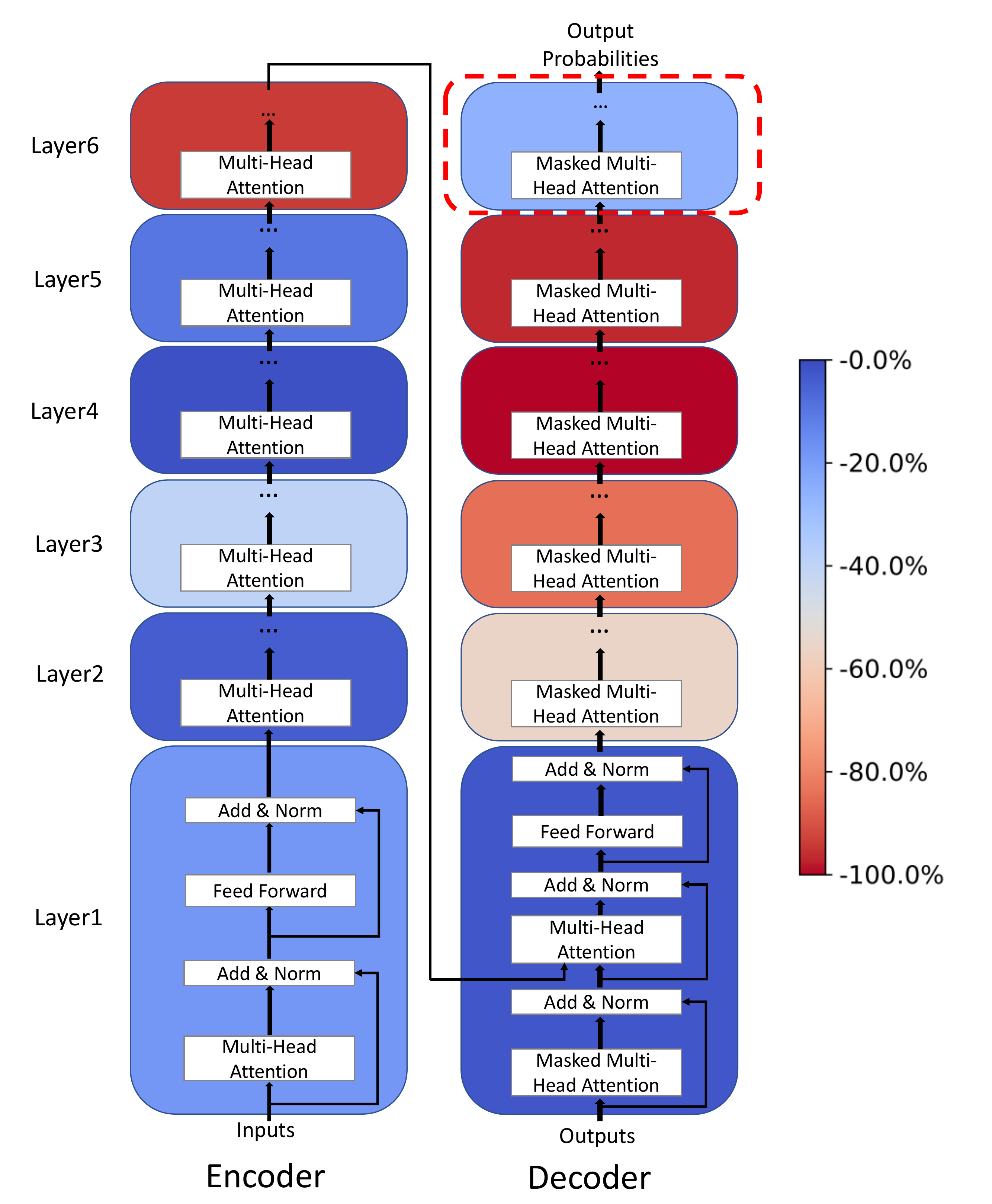}}
\hfil
\caption{
Visualization of different Transformer layers with different defense settings under parameter corruptions. Not all layers are enhanced when defending all parameters. However, defending a specific layer can improve the robustness of this layer.
}
\label{fig:layer-def}
\end{figure*}

\subsection{Enhancing Pretrained Language Model}

We further verify the effectiveness of the proposed defense method on BERT~\citep{Bert}. BERT~\citep{Bert} is a transformer-based pretrained language model, which can be adopted in downstream tasks after finetuning. We finetune a BERT on the SST-2~\citep{SST-2} (Stanford Sentiment Treebank) sentiment classification dataset. We do not consider the setting of the start epoch because BERT is already pretrained. Detailed settings are in Appendix A. The experimental results are in Table~\ref{tab:bert}. The training process does not converge with defense methods in previous researches, while the proposed defense method can in the end enhance pretrained language models.

\subsection{Hypothesis Testing}
Further, we conduct hypothesis testings on the CIFAR-10, VOC, En-Vi, De-En, and SST-2 datasets to testify whether: (1) The proposed defense method outperforms the baselines; and (2) The proposed defense method outperforms its variants. The results show that on all datasets, our proposed defense method outperforms the baselines and its variants significantly $(p<0.05)$. Besides, on CIFAR-10, ACRT~\citep{attack-paper} and AWP~\citep{Adversarial_data_Weight_Perturbation} outperform the baselines significantly $(p<0.05)$. However, on VOC, En-Vi, and De-En, they do not outperform the baselines significantly $(p<0.05)$. On the SST-2 dataset, the loss does not converge. Please refer to Appendix C for details.

\subsection{Probing Different Layers of DNN}

The parameter corruption algorithm can be utilized to probe the robustness of different groups of parameters. The Transformer models in our experiments can be divided into 6 encoder layers and 6 decoder layers. We probe and visualize the vulnerability of different layers of Transformer via multi-step adversarial parameter corruption as in Figure~\ref{fig:layer-def}. In Figure~\ref{fig:layer-def} (a\&d), we can see that the higher layers in the Transformer encoder or decoder are less robust to parameter corruptions. It is possible because the changes in the output of lower layers due to its parameter corruption will then be amended by the normalization layers, making the model more robust to perturbation. 

In Figure~\ref{fig:layer-def} (b\&e), we can see that the defense method improves the robustness of the whole layers. However, some high layers are still sensitive to parameter perturbations. If we want to enhance a certain layer in the model, the defense method may not work efficiently. To solve this issue, we propose to defend locally for a certain layer. Take the last decoder layer of Transformer for example: we set $\vect{w}$ as parameters in the last decoder layer and $\bm\theta$ as other parameters. The experimental results are shown in Table~\ref{tab:part-def} and Figure~\ref{fig:layer-def} (c\&f). It can be seen that the overall accuracy and robustness are improved compared to the baseline. As expected, the last decoder layer is significantly enhanced in terms of parameter robustness. Therefore, defending a certain group of parameters localized can enhance the targeted group of parameters contrapuntally.

\begin{table}[!t]
        \centering
        \tiny
        \setlength{\tabcolsep}{3 pt}
        \caption{Results of defending the last decoder layer in Transformer, compared to other settings. Corrupt whole denotes that all parameters can be corrupted. Corrupt last layer denotes that the parameters in the last layer in the transformer decoder can be corrupted.}
        \label{tab:part-def}
        \begin{tabular}{@{}lccc@{}}
            \toprule
             \multicolumn{4}{c}{\textbf{En-Vi (BLEU)}} \\
             \midrule
              & w/o Corruption & Corrupt whole & Corrupt last layer \\
            \midrule
              w/o defense & 30.64 & 16.54 & 0.98 \\
             Defend all layers & 31.09 & \textbf{30.71} & 17.36 \\
             Defend the last layer & \textbf{31.24} & 12.34 & \textbf{30.01} \\
             \midrule
             \multicolumn{4}{c}{\textbf{De-En (BLEU)}}\\ 
             \midrule
              & w/o Corruption & Corrupt whole &  Corrupt last layer \\
             \midrule
              w/o defense & 35.32 & 18.30 & 0.00\\
             Defend all layers & \textbf{35.88} & \textbf{35.72} & 0.82\\
             Defend the last layer & 35.44 & 18.10 & \textbf{22.45} \\
             \bottomrule
        \end{tabular}
\end{table}

\section{Related Work}

\subsection{Adversarial Examples and Adversarial Training} 

\cite{Intriguing_properties_of_neural_networks} first proposed the concept of adversarial examples and found that neural networks are vulnerable to adversarial attacks on input data. Adversarial examples can mislead neural networks even in physical world scenarios, such as making small perturbations to signals from cameras as the input data~\citep{Adversarial_examples_in_the_physical_world}. A line of existing studies concerning adversarial attacks on neural networks focused on generating adversarial examples~\citep{Explaining_and_Harnessing_Adversarial_Examples,Adversarial_examples_in_the_physical_world,Deepfool}. Other related work on adversarial examples aimed to design adversarial training algorithms with respect to adversarial examples to improve the robustness of neural networks over adversarial examples~\citep{Explaining_and_Harnessing_Adversarial_Examples,Towards_Evaluating_the_Robustness_of_Neural_Networks,YOPO,freeLB,Unified-min-max}.

\cc{Besides adversarial training algorithms,} \cite{related_2_Stochastic_Activation_Pruning} proposed stochastic activation pruning to defense adversarial examples, and \cite{related_6_enforcing_Feature-Consistency-Across-Bit-Planes} proposed to enforce the consistency of features on the bit planes for better adversarial robustness. \cite{related_4_corruption_rewards_transition_probabilities, related_5_multi-armed_bandits} considered adversarial corruptions in reinforcement learning, such as corruptions in rewards or transition probabilities. \cite{related_1_interpolating_function,related_3_Maximum-Entropy-Adversarial-Data-Augmentation} adopted mechanisms similar to adversarial training for adversarial data augmentation.

\subsection{Parameter Robustness in Neural Networks}

Existing studies also concerned changes~\citep{LCA} or perturbations in network parameters, caused by  training data poisoning~\citep{badnet, backdoor1, backdoor2, backdoor-Bert}, bit flipping~\citep{TBT}, compression~\citep{Stronger_Generalization_Compression} or parameter quantization~\citep{Data-Free-Quant,tensorRT,parameter_L1}. To drive the parameters from the sharp minima and improve the parameter robustness, existing researches adopted gradient regularization~\citep{parameter_L1,attack-paper,Adversarial_Defense_gradient_flatness} or attempted to minimize the risk of adversarial parameter perturbation with an estimated optimal parameter corruption. The optimal parameter corruption can be generated by the gradient-based corruption~\citep{attack-paper}, the PGD algorithm~\citep{Regularizing_NN_via_Adversarial_Perturbation_L2,Sharpness-Aware_Minimization}, or the PGD algorithm combined with adversarial examples~\citep{Adversarial_data_Weight_Perturbation}. Different from these studies, we estimate the risk with multi-step parameter corruptions and average the risks for multiple parameter corruptions as an estimation of the risk.

\section{Conclusion}
In this work, we introduce the concept of parameter corruption and propose the multi-step parameter corruption algorithm for probing the parameter robustness of neural networks. To enhance neural networks, we propose the adversarial parameter defense algorithm that minimizes the risk of parameter corruption based on multi-step risk estimation. Experimental results show that our proposed defense algorithm can improve both the accuracy and the robustness of neural networks under multiple parameter corruption methods, including parameter corruptions, random noises, or real-world quantization. 

\section*{Acknowledgments}
This work is partly supported  by National Key R\&D Program of China No. 2019YFC1521200. This work is also partly supported by Beijing Academy of Artificial Intelligence (BAAI). 
Xu Sun and Qi Su are the corresponding authors.

\bibliography{mybibfile}
\onecolumn
\appendix

\section{Model Implementation}
This section shows the implementation details of neural networks used in our experiments. Experiments are conducted on a GeForce GTX TITAN X GPU.

\subsection{ResNet}

\subsubsection{CIFAR-10}

CIFAR-10\footnote{CIFAR-10 can be found at \url{https://www.cs.toronto.edu/~kriz/cifar.html}} \cite{CIFAR-10} is an image classification dataset with 10 categories and consists of 50,000 training images and 10,000 test images. The images are of 32-by-32 pixel size with 3 channels.

For CIFAR-10, we implement Resnet-101~\cite{resnet} as baseline. The learning rate is 0.1, the weight decay is $5\times 10^{-4}$ and momentum is 0.9, and batch size is 128. The optimizer is SGD. We train the model for 200 epochs. After 150 epochs, the learning rate is adjusted to $0.01$. 

In defense, the start epoch is 10, $K=1$, $\epsilon=0.1$, and we adopt $L_2$ constraint. In ACRT~\citep{attack-paper} or SAM~\citep{Sharpness-Aware_Minimization}, the start epoch is 10, $\epsilon=0.1$, and we adopt $L_2$ constraint. In AWP~\citep{Adversarial_data_Weight_Perturbation}, the start epoch is 10, when generating virtual parameter corruption $K=3$, $\epsilon=0.1$, and we adopt $L_2$ constraint.
 
\subsubsection{VOC}
PASCAL VOC 2007\footnote{VOC can be found at \url{http://www.pascal-network.org/challenges/VOC/voc2007/index.html}}\cite{VOC} is an object detection dataset, it consists of 5,000 train-valid images and 5,000 test images over object categories.

For VOC, we implement Faster-RCNN\footnote{The implementation of baseline Faster-RCNN can be found at \url{https://github.com/jwyang/faster-rcnn.pytorch}}~\cite{faster-rcnn} and Resnet-101~\cite{resnet} as baseline. The learning rate is 0.001, the batch size is 1. The optimizer is SGD. We train the model for 7 epochs. After 5 epochs, the learning rate is adjusted to 0.0001. 

In defense, the start epoch is 1, $K=2$, $\epsilon=0.02$, and we adopt $L_2$ constraint. In ACRT~\citep{attack-paper} or SAM~\citep{Sharpness-Aware_Minimization}, the start epoch is 1, $\epsilon=0.04$, and we adopt $L_2$ constraint. In AWP~\citep{Adversarial_data_Weight_Perturbation}, the start epoch is 1, when generating virtual parameter corruption $K=3$, $\epsilon=0.04$, and we adopt $L_2$ constraint.

\subsection{Transformer}

\subsubsection{De-En}
The De-En dataset is provided by the IWSLT 2014 Evaluation Campaign~\citep{2014iwslt}. We use the same dataset splits following previous work~\citep{fairseq,seq2seqRNN,beam}. It contains 153K sentences for training, 7K sentences for validation, and 7K sentences for testing. BPE is used to get vocabulary. We use the shared embedding setting and the vocabulary size is 10,149. 

We use ``transformer\_iwslt\_de\_en''  provided by fairseq\footnote{Both implementations of baseline Transformers and two datasets can be found at \url{https://github.com/pytorch/fairseq}}\cite{fairseq} as our basic model. We use optimizer Adam with $\beta_1 = 0.9$ and $\beta_2=0.98$. The dropout rate is 0.3. The attention dropout rate is 0.1. The activation dropout is 0.1. The initialization learning rate is $10^{-7}$ and the learning rate is 0.0015. The training batch size is  4,096 tokens. We update gradients for every 2 steps. We train the model for 70 epochs, and the number of warmup steps is 8,000. We average the last 10 checkpoints for evaluation and set the beam size to 5.

In defense, the start epoch is 30, $K=2$, $\epsilon=0.0006$ and we adopt $L_{+\infty}$ constraint. When defending the last layer of the decoder, the start epoch is 30, $K=1$, $\epsilon=0.002$ and we adopt $L_{+\infty}$ constraint. In ACRT~\citep{attack-paper} or SAM~\citep{Sharpness-Aware_Minimization}, the start epoch is 30, $\epsilon=0.0004$, and we adopt $L_{+\infty}$ constraint. In AWP~\citep{Adversarial_data_Weight_Perturbation}, the start epoch is 10, when generating virtual parameter corruption $K=2$, $\epsilon=0.0004$, and we adopt $L_{+\infty}$ constraint.

In Figure 3 (d), (e) and (f), we set $\epsilon=0.02$. In Table 5, we set $\epsilon=0.005$, $n=100$ and adopt $L_{+\infty}$ constraint when corrupting the whole model, and we set $\epsilon=0.02$, $n=100$ and adopt $L_{+\infty}$ constraint when corrupting the last decoder layer.

\subsubsection{En-Vi}
The En-Vi dataset contains 133K training sentence pairs provided by the IWSLT 2015 Evaluation Campaign~\citep{2015iwslt}. We use TED
tst2012 (1,553 sentences) as the validation set and TED tst2013 (1,268 sentences) as the test set. BPE is used to get input and output vocabulary. The English and Vietnamese vocabulary sizes are 7,669 and 6,669 respectively. 

We use ``transformer\_wmt\_en\_de''  provided by fairseq\cite{fairseq} as our basic model. We use optimizer Adam with $\beta_1 = 0.9$ and $\beta_2= 0.98$. The dropout rate is 0.1. The learning rate is 0.001. The training batch size is  4,096 tokens. We train the model for 52 epochs, and the number of warmup steps is 8,000. We average the last 10 checkpoints for evaluation and set the beam size to 5. 

In defense, the start epoch is 35, $K=2$, $\epsilon=0.00045$ and we adopt $L_{+\infty}$ constraint. When defending the last layer of the decoder, the start epoch is 35, $K=3$, $\epsilon=0.0015$ and we adopt $L_{+\infty}$ constraint. In ACRT~\citep{attack-paper} or SAM~\citep{Sharpness-Aware_Minimization}, the start epoch is 30, $\epsilon=0.0012$, and we adopt $L_{+\infty}$ constraint. In AWP~\citep{Adversarial_data_Weight_Perturbation}, the start epoch is 30, when generating virtual parameter corruption $K=2$, $\epsilon=0.0012$, and we adopt $L_{+\infty}$ constraint.

In Figure 3 (d), (e) and (f), we set $\epsilon=0.01$. In Table 5, we set $\epsilon=0.002$, $n=100$ and adopt $L_{+\infty}$ constraint when corrupting the whole model, and we set $\epsilon=0.01$, $n=100$ and adopt $L_{+\infty}$ constraint when corrupting the last decoder layer.

\subsection{BERT}

SST-2~\citep{SST-2} is the Stanford Sentiment Treebank classification dataset with two classes. It includes 63,750 training sentences, 873 development sentences, and 1,820 test sentences. In our paper, we adopt the development sentences as the test set. The sentences are preprocessed to lowercased and tokenized by the uncased BERT tokenizer. Lengths of sentences are truncated to 128 tokens (including special tokens).

We fine-tune the uncased BERT base model~\citep{Bert}\footnote{The implementation~\citep{huggingface-BERT} can be found at \url{https://github.com/huggingface/transformers}} as our basic model. We adopt the AdamW optimizer. The training batch size is 32 and the learning rate is $5\times 10^{-5}$. We fine-tuning the model for 3 epochs. 

In defense, the start epoch is 0, $K=1$, $\epsilon=0.5$ and we adopt $L_{+\infty}$ constraint.  In ACRT~\citep{attack-paper} or SAM~\citep{Sharpness-Aware_Minimization}, and AWP~\citep{Adversarial_data_Weight_Perturbation}, we grid search many settings, but the loss always diverges.

\subsection{Details of Corruption Approaches}

To verify the robustness of models, baseline models and models with defense algorithms are tested by four testing corruption approaches: (1) Our proposed multi-step adversarial corruption method; (2) Our proposed gradient-based corruption method; (3) Random Gaussian or uniform noises on parameters to simulate random corruptions; (4) Tensor-RT~\citep{tensorRT} weight quantization method, which quantifies parameters into $n$-bit signed integers.

In the gradient-based corruption method~\citep{attack-paper}, all parameters can be corrupted and the number of corrupted parameters are not limited. The weight quantization method is a uniform scheme following Tensor-RT~\citep{tensorRT}. The parameters are divided into parameter groups, where a group is usually a weight matrix or bias vector. For a parameter group with the floating-point base $W_\text{float}$, the quantized $n$-bit fixed-point (signed integer) counterpart $W_\text{fixed}$ is:
\begin{equation}
    w_0=\frac{\max(|W_\text{float}|)}{2^{n-1}-1},\quad W_\text{fixed} = \text{round}(\frac{W_\text{float}}{w_0}) \cdot w_0.
\end{equation}

\section{Computational Complexity of Our Defense Algorithm}

Our defense algorithm involves $K+1$ forward and backward propagation and $K$ times to generating new corruptions. Generating a new corruption should be trivial because we can solve it by the closed-form solution given the gradients in backward propagation.

First, consider the time complexity of our defense algorithm. Suppose $T_f$ and $T_b$ denote the forward and backward time cost of ordinary training of the baseline model. Then the time complexity of our defense algorithm is approximate $(K+1)(T_f+T_b)$, which is about $K+1$ times of baseline.

Suppose $M_f$ and $M_b$ denote the forward and backward memory cost of ordinary training of the baseline model. In our defense algorithm, we do not need to save all gradients in $K+1$ steps (including step $0$). Our goal is to obtain the average gradient in $K+1$ steps. Therefore, we can save the memory cost by saving only the partial sum of gradients for the final optimization and the gradient in the current step. After every forward propagation, we can free the memory occupied. Therefore, the time complexity of our defense algorithm is approximately $M_f+2M_b$, which is less than two times of baseline.

\section{Hypothesis test}

\begin{table}[!t]
\caption{Details of hypothesis testings. $t$-values marked with $^\#$ are smaller than $2.132$ and namely not statistically significantly $(p<0.05)$.}
\label{tab:t-test}
\scriptsize
\setlength{\tabcolsep}{2pt}
\centering
\begin{tabular}{@{}lccccc@{}}
\toprule
 \bf Datasest & \textbf{CIFAR-10} &  \textbf{VOC} &  \textbf{En-Vi} & \textbf{De-En} & \textbf{SST-2}\\ 
 \midrule
 Baseline & 94.46$\pm$0.164 & 74.90$\pm$0.200 & 30.64$\pm$0.015 & 35.32$\pm$0.131 & 92.03$\pm$0.55 \\
 \midrule
 ACRT~\citep{attack-paper} & 96.23$\pm$0.031 & 75.13$\pm$0.306 & 30.71$\pm$0.053 & 35.53$\pm$0.146 & $-\infty$ (not converge)\\
 AWP~\citep{Adversarial_data_Weight_Perturbation} & 96.08$\pm$0.093 & 75.03$\pm$0.153 & 30.62$\pm$0.167 & 35.49$\pm$0.229 & $-\infty$ (not converge)\\
 \midrule
 Proposed & \textbf{96.34$\pm$0.076} & \textbf{75.77$\pm$0.152} & \textbf{31.09$\pm$0.102} & \textbf{35.88$\pm$0.053} &  \textbf{92.78$\pm$0.18}\\
 \midrule
 \textbf{Null Hypothesis} & \multicolumn{5}{c}{\textbf{$t$-values}} \\ 
 \midrule
 Proposed$\le$Baseline & 18.01 & 6.00 & 7.55 & 6.88 &  2.24 \\
 \midrule
 Proposed$\le$ACRT & 2.32 & 5.95 & 4.16 &2.88 &  $+\infty$ \\
 Proposed$\le$AWP & 3.17 & 3.87 &  5.72 & 3.91 & $+\infty$ \\
 \midrule
 ACRT$\le$Baseline & 14.88 & 0.89$^\#$ &  2.20 & 1.12$^\#$ & $-\infty^\#$ \\
 AWP$\le$Baseline &  18.37 & 3.06 & -0.21$^\#$& 1.86$^\#$ & $-\infty^\#$ \\
\bottomrule
\end{tabular}
\end{table}
In this section, we conduct hypothesis testings on the CIFAR-10, VOC, En-Vi, De-En, and SST-2 datasets to verify that: (1) Our proposed defense method outperforms the baselines; and (2) Our proposed defense method outperforms its variants. 

We conduct the student-$t$ test and the time of repeated experiments $n=3$. $df=2(n-1)=4$, when $p=0.05$, $t\text{-value}=2.132$. The details of hypothesis testings are shown in Table~\ref{tab:t-test}.

The results show that on all datasets, our proposed defense method outperforms the baselines and its variants statistically significantly $(p<0.05)$. Besides, on CIFAR-10, ACRT~\citep{attack-paper} and AWP~\citep{Adversarial_data_Weight_Perturbation} outperform the baselines statistically significantly $(p<0.05)$. However, on VOC, En-Vi and De-En, they do not outperform the baselines statistically significantly $(p<0.05)$. On the SST-2 dataset, they do not converge.

\section{Theoretical Analysis}

\subsection{Proof of Theorem~\ref{thmA:random}}
\begin{thmA}[Distribution of Random Corruption]
\label{thmA:random}
Given the constraint set  $S=\{\vect{a}:\|\vect{a}\|_2= \epsilon\}$ and a generated random corruption $\vect{\tilde a}$, which in turn obeys a uniform distribution on $\|\vect{\tilde a}\|_2=\epsilon$. The estimation of $\Delta_\text{ave}\mathcal{L}(\vect{w}, S, \mathcal{D})$ and $\Delta_\text{max}\mathcal{L}(\vect{w}, S, \mathcal{D})$ are:
\begin{align}
\Delta_\text{ave}\mathcal{L}(\vect{w}, S, \mathcal{D})&=\frac{tr(\textbf{H})}{2k}\epsilon^2+o(\epsilon^2),\\
\Delta_\text{max}\mathcal{L}(\vect{w}, S, \mathcal{D})&=\epsilon G+o(\epsilon). \label{eqA:expectation}
\end{align}
Define $\eta=\nicefrac{|\vect{\tilde a}^\text{T}\vect{g}|}{\epsilon G}$, which is a first-order estimation of  $\nicefrac{|\Delta\mathcal{L}(\vect{w}, \vect{\tilde a}, \mathcal{D})| }{\Delta_\text{max}\mathcal{L}(\vect{w}, S, \mathcal{D})}$ and $\eta\in [0, 1]$, then the probability density function $p_\eta(x)$ of $\eta$ and the cumulative density $P(\eta \le x)$ function of $\eta$ are:
\begin{align}
p_{\eta}(x)&=\frac{2\Gamma(\frac{k}{2})}{\sqrt{\pi}\Gamma(\frac{k-1}{2})}(1-x^2)^{\frac{k-3}{2}}, \label{equA:random_destiny1}\\
P(\eta \le x)&=\frac{2xF_1(\frac{1}{2}, \frac{3-k}{2};\frac{3}{2}; x^2)}{B(\frac{k-1}{2}, \frac{1}{2})}, \label{equA:random_destiny2}
\end{align}
where $k$ denotes the number of corrupted parameters, and $\Gamma(\cdot)$, $B(\cdot,\cdot)$ and $F_1(\cdot,\cdot;\cdot;\cdot)$ denote the gamma function, beta function and hyper-geometric function.
\end{thmA}

The detailed definitions of the gamma function, beta function and hyper-geometric function are as follows: $\Gamma(\cdot)$ and $B(\cdot,\cdot)$ denote the gamma function and beta function, and $F_1(\cdot,\cdot;\cdot;\cdot)$ denotes the Gaussian or ordinary hyper-geometric function, which can also be written as $ _2 F_1(\cdot,\cdot;\cdot;\cdot)$:
\begin{align}
\Gamma(z)&=\int_{0}^{+\infty}t^{z-1}e^{-t}dt,\\  B(p, q)&=\int_{0}^{1}t^{p-1}(1-t)^{q-1}dt,\\
F_1(a,b;c;z)&=1+\sum\limits_{n=1}^{+\infty}\frac{a(a+1)\cdots(a+n-1)\times b(b+1)\cdots(b+n-1)}{c(c+1)\cdots(c+n-1)}\frac{z^n}{n!}.
\end{align}

\begin{proof}

First, We will prove Eq.(\ref{eqA:expectation}). Note that $\vect{\tilde  a}$ obeys a uniform distribution on $\|\vect{\tilde a}\|_2=\epsilon$,
\begin{align}
\Delta\mathcal{L}(\vect{w},\vect{a}; \mathcal{D}) &= \vect{a}^\text{T}\vect{g}+\frac{1}{2}\vect{a}^\text{T}\textbf{H}\vect{a}+o(\epsilon^2)=\vect{a}^\text{T}\vect{g}+o(\epsilon), \\
\max\limits_{\|\vect{a}\|_2=\epsilon} \vect{a}^\text{T}\vect{g}&=\epsilon G.
\end{align}

Therefore,
\begin{align}
\Delta_\text{max}\mathcal{L}(\vect{w}, S, \mathcal{D})=\epsilon G+o(\epsilon).
\end{align}

Suppose $\vect{\tilde a}=(a_1, a_2, \cdots,a_{k-1}, a_{k})^\text{T}, \vect{g}=(g_1, g_2, \cdots,g_{k-1}, g_{k})^\text{T}$ and  $H_{ij}=\nicefrac{\partial^2\mathcal{L}(\vect{w}+\vect{a};\mathcal{D})}{\partial a_i \partial a_j}$. Since  $\vect{\tilde  a}$ obeys a uniform distribution on $\|\vect{\tilde a}\|_2=\epsilon$, by symmetry, we have, 
\begin{align}
\mathbb{E}_{\|\vect{\tilde a}\|_2=\epsilon}[a_i]&=
\mathbb{E}_{\|\vect{\tilde a}\|_2=\epsilon}[a_ia_j]=0\ (i\ne j)\\
\mathbb{E}_{\|\vect{\tilde a}\|_2=\epsilon}[a_i^2]&=
\mathbb{E}_{\|\vect{\tilde a}\|_2=\epsilon}[\frac{\|\vect{a}\|^2}{k}]=\frac{\epsilon^2}{k}.
\end{align}

Therefore,
\begin{align}
\mathbb{E}_{\|\vect{\tilde a}\|_2=\epsilon}[\Delta\mathcal{L}(\vect{w}, \vect{\tilde a}, \mathcal{D})] &= \mathbb{E}_{\|\vect{\tilde a}\|_2=\epsilon}[\vect{\tilde a}^\text{T}\vect{g}+\frac{1}{2}\vect{\tilde a}^\text{T}\textbf{H}\vect{\tilde a}+o(\epsilon^2)] \\
&=\mathbb{E}_{\|\vect{\tilde a}\|_2=\epsilon}[\vect{\tilde a}^\text{T}\vect{g}]+\mathbb{E}_{\|\vect{\tilde a}\|_2=\epsilon}[\frac{1}{2}\vect{\tilde a}^\text{T}\textbf{H}\vect{\tilde a}]+o(\epsilon^2)\\
&=\mathbb{E}_{\|\vect{\tilde a}\|_2=\epsilon}[\sum\limits_{i}g_ia_i]+\mathbb{E}_{\|\vect{\tilde a}\|_2=\epsilon}[\frac{1}{2}\sum\limits_{i, j}H_{ij}a_ia_j]+o(\epsilon^2)\\
&=\sum\limits_{i}H_{ii}\frac{\epsilon^2}{2k}+o(\epsilon^2)\\
&=\frac{\text{trace}(\textbf{H})}{2k}\epsilon^2+o(\epsilon^2).
\end{align}

Then, we will prove Eq.(\ref{equA:random_destiny2}). Because of the rotational invariance of the distribution of $\vect{\tilde  a}$, we may assume $\frac{\vect{g}}{\|\vect{g}\|_2}=(1, 0, 0, \cdots, 0)^\text{T}, \vect{\tilde  a}=(a_1, a_2, a_3,\cdots,a_{k-1}, a_{k})^\text{T}$ and,
\begin{equation} 
\left \{
\begin{aligned} 
a_1 &= \epsilon\cos\phi_1\\ 
a_2 &= \epsilon\sin\phi_1\cos\phi_2\\
a_3 &= \epsilon\sin\phi_1\sin\phi_2\cos\phi_3\\
&\quad \cdots\\
a_{k-1} &= \epsilon\sin\phi_1\sin\phi_2\cdots\sin\phi_{k-2}\cos\phi_{k-1}\\
a_{k} &= \epsilon\sin\phi_1\sin\phi_2\cdots\sin\phi_{k-2}\sin\phi_{k-1}\\
\end{aligned} 
\right.
;
\end{equation}
where $\phi_i\in[0, \pi]\ (i\ne k-1)$ and $\phi_{k-1}\in[0, 2\pi)$. For $x \in [0, 1]$, define $\alpha = \arccos x$, then:
\begin{equation}
f(\vect{\tilde  a})=\vect{\tilde  a}^\text{T}\vect{g}=\epsilon G\cos\phi_1, P(\eta \le x)=P(|\cos\phi_1| \le x)=2P(0\le\phi_1\le \alpha).
\end{equation}

That is to say,
\begin{align}
P(\eta\le x)&=\frac{2\int_0^{2\pi}\int_0^{\pi}\cdots\int_0^{\alpha}(\sin^{k-2}\phi_{1}\sin^{k-3}\phi_{2}\cdots \sin\phi_{k-2}) d\phi_1\cdots  d\phi_{k-2}d\phi_{k-1}}{\int_0^{2\pi}\int_0^{\pi}\cdots\int_0^{\pi}(\sin^{k-2}\phi_{1}\sin^{k-3}\phi_{2}\cdots \sin\phi_{k-2}) d\phi_1\cdots  d\phi_{k-2}d\phi_{k-1}}\\
&=\frac{2\int_0^{\alpha}\sin^{k-2}\phi_{1}d\phi_1}{\int_0^{\pi}\sin^{k-2}\phi_{1}d\phi_1}=\frac{\int_0^{\alpha}\sin^{k-2}\phi_{1}d\phi_1}{\int_0^{\frac{\pi}{2}}\sin^{k-2}\phi_{1}d\phi_1}=\frac{2\int_0^{\alpha}\sin^{k-2}\phi_{1}d\phi_1}{B(\frac{k-1}{2}, \frac{1}{2})} \label{equA:P}\\
&=\frac{2\cos\alpha F_1(\frac{1}{2}, \frac{3-k}{2};\frac{3}{2}; \cos^2\alpha)}{B(\frac{k-1}{2}, \frac{1}{2})}=\frac{2xF_1(\frac{1}{2}, \frac{3-k}{2};\frac{3}{2}; x^2)}{B(\frac{k-1}{2}, \frac{1}{2})},
\end{align}
and notice that:
\begin{align}
\sin\alpha=(1-x^2)^\frac{1}{2}, \big|\frac{d\alpha}{dx}\big|=\frac{1}{(1-x^2)^\frac{1}{2}}, B(p, q)=\frac{\Gamma(p)\Gamma(q)}{\Gamma(p+q)}, \Gamma(\frac{1}{2})=\sqrt{\pi},
\end{align}
then according to Eq.(\ref{equA:P}):
\begin{align}
p_\eta(x)=\frac{2\sin^{k-2}\alpha}{B(\frac{k-1}{2}, \frac{1}{2})}\big|\frac{d\alpha}{dx}\big|=\frac{2\Gamma(\frac{k}{2})}{\sqrt{\pi}\Gamma(\frac{k-1}{2})}(1-x^2)^{\frac{k-3}{2}}.
\end{align}
\end{proof}

\subsection{Proof of Theorem~\ref{thmA:bound}}
\begin{thmA}[Error Bound of the Gradient-Based Estimation]
\label{thmA:bound}
Suppose $\mathcal{L}(\vect{w};\mathcal{D})$ is convex and $L$-smooth with respect to $\vect{w}$ in the subspace $\{\vect{w}+\vect{a}:\vect{a}\in S\}$, where $S=\{\vect{a}:\|\vect{a}\|_p=\epsilon\text{ and }\|\vect{a}\|_0\le n\}$.\footnote{Note that $\mathcal{L}$ is only required to be convex and $L$-smooth in a neighbourhood of $\vect{w}$, instead of the entire $\mathbb{R}^k$.} Suppose $\vect{a^*}$ and $\vect{\hat a}$ are the optimal corruption and the gradient-based corruption in $S$ respectively. $\|\vect{g}\|_2=G>0$. It is easy to verify that $\mathcal{L}(\vect{w}+\vect{a^*};\mathcal{D})\ge \mathcal{L}(\vect{w+\vect{\hat a}};\mathcal{D})>\mathcal{L}(\vect{w};\mathcal{D})$ . It can be proved that the loss change of the gradient-based corruption is the same order infinitesimal of that of the optimal parameter corruption:
\begin{equation}
\frac{\Delta_\text{max}\mathcal{L}(\vect{w}, S; \mathcal{D})}{\Delta\mathcal{L}(\vect{w}, \vect{\hat a}; \mathcal{D})}=1+O\left(\frac{Ln^{g(p)}\sqrt{k}\epsilon}{G}\right);
\label{eqA:bound}
\end{equation}
where $g(p)$ is formulated as $g(p)=\max\{\frac{p-4}{2p}, \frac{1-p}{p}\}$.
\end{thmA}
\begin{proof}
Define $q=\frac{p}{p-1}, \frac{1}{p}+\frac{1}{q}=1$ here. 

We introduce a lemma first.
\begin{lem}
For vector $x\in \mathbb{R}^k$, $\|\vect{x}\|_0\le n \le k$, for any $r>1$, $\|\vect{x}\|_2\le \beta_r \|\vect{x}\|_r$, where $\beta_r=\max\{1, n^{1/2-1/r}\}$. 
\label{lemma:norm1}
\end{lem}
\begin{proof}[Proof of Lemma~\ref{lemma:norm1}]
We may assume $\vect{x}=(x_1, x_2, \cdots, x_k)^\text{T}$ and $x_{n+1}=x_{n+2}=\cdots=x_{k}=0$. Then $\|\vect{x}\|_r=\big(\sum\limits_{i=1}^n|x_i|^r\big)^{\frac{1}{r}}$.

When $1<r<2$, define $t=\frac{r}{2}<1$ and $h(x)=x^t+(1-x)^t$, $h''(x)=t(t-1)(x^{t-2}+(1-x)^{t-2})<0$, thus $h(x)\ge \max\{h(0), h(1)\}=1\ (x\in[0, 1])$. 

Then for $a, b\ge 0$ and $a+b>0$, we have $\frac{a^t+b^t}{(a+b)^t}=(\frac{a}{a+b})^t+(1-\frac{a}{a+b})^t=h(\frac{a}{a+b})\ge 1$. That is to say, $a^t+b^t \ge (a+b)^t$. More generally, $a^t+b^t+\cdots+c^t \ge (a+b+\cdots+c)^t$. Therefore,
\begin{align}
\|\vect{x}\|_r=\big(\sum\limits_{i=1}^n|x_i|^{r}\big)^{\frac{1}{r}}=\big(\sum\limits_{i=1}^n(|x_i|^2)^\frac{r}{2}\big)^{\frac{1}{r}} \ge \big((\sum\limits_{i=1}^n|x_i|^2)^\frac{r}{2}\big)^{\frac{1}{r}} = \|\vect{x}\|_2.
\end{align}

When $r\ge 2$, according to the power mean inequality,
\begin{align}
\|\vect{x}\|_r=\big(\sum\limits_{i=1}^n|x_i|^{r}\big)^{\frac{1}{r}}=n^\frac{1}{r} \big(\frac{\sum\limits_{i=1}^n|x_i|^r}{n}\big)^{\frac{1}{r}} \ge n^\frac{1}{r} \big(\frac{\sum\limits_{i=1}^n|x_i|^2}{n}\big)^{\frac{1}{2}}=  n^{\frac{1}{r}-\frac{1}{2}}(\sum\limits_{i=1}^n|x_i|^2)^\frac{1}{2} = n^{\frac{1}{r}-\frac{1}{2}}\|\vect{x}\|_2.
\end{align}

To conclude, $\|\vect{x}\|_2\le \beta_r \|\vect{x}\|_r$, where $\beta_r=\max\{1, n^{1/2-1/r}\}$.
\end{proof}

According to Lemma~\ref{lemma:norm1}, notice that $\|\vect{a}^*\|_0\le n$, define $\vect{h}=\text{top}_n(\vect{g})$, then  $\|\vect{h}\|_2\ge\frac{n}{k}\|\vect{g}\|_2$ we have, 
\begin{align}
\|\vect{a}^*\|_2 \le \beta_p\|\vect{a}^*\|_p\le \beta_p\epsilon
,\quad 
\|\vect{h}\|_q\ge\frac{\|\vect{h}\|_2}{\beta_q}\ge \frac{\|\vect{g}\|_2}{\beta_q}\sqrt{\frac{n}{k}}= \frac{G}{\beta_q}\sqrt{\frac{n}{k}}.
\end{align}

Since $\mathcal{L}(\vect{w};\mathcal{D})$ is convex and $L$-smooth in $\vect{w}+S$,
\begin{align}
\Delta\mathcal{L}(\vect{w}, \vect{\hat a}, \mathcal{D}))&\ge \vect{g}^\text{T}\vect{\hat a}=\epsilon \|\vect{h}\|_q\\
\Delta\mathcal{L}(\vect{w}, \vect{a^*}, \mathcal{D})&\le \vect{g}^\text{T}\vect{a}^*+\frac{L}{2}\|\vect{a}^*\|_2^2 = \epsilon \|\vect{h}\|_q+\frac{L}{2}\|\vect{a}^*\|_2^2. \end{align}

Therefore,
\begin{align}
\text{Left Hand Side}=\frac{\Delta\mathcal{L}(\vect{w}, \vect{a^*}, \mathcal{D})}{\Delta\mathcal{L}(\vect{w}, \vect{\hat a}, \mathcal{D})}
\le\frac{\epsilon \|\vect{h}\|_q+\frac{L}{2}\|\vect{a}^*\|_2^2}{\epsilon \|\vect{h}\|_q}
\le1+\frac{L\beta_p^2\epsilon}{2\|\vect{h}\|_q}
\le1+\frac{L\beta_p^2\beta_q\epsilon\sqrt{k}}{2G\sqrt{n}}.
\end{align}

When $p\ge 2,q\le 2$,  $\beta_p^2\beta_q=n^{1-2/p}$, and when $p\le 2,q\ge 2$, $\beta_p^2\beta_q=n^{1/2-1/q}=n^{1/p-1/2}$. To conclude, $\beta_p^2\beta_q=\max\{n^{1-2/p}, n^{1/p-1/2}\}=n^{\max\{1-2/p, 1/p-1/2\}}$. Therefore,
\begin{align}
\text{Left Hand Side}\le 1+\frac{Ln^{\max\{1-2/p, 1/p-1/2\}}\sqrt{k}}{2G\sqrt{n}}\epsilon = 1+O\left(\frac{Ln^{g(p)}\sqrt{k}\epsilon}{G}\right),
\end{align}
where $g(p)=\max\{\frac{p-4}{2p}, \frac{1-p}{p}\}$.
\end{proof}

\subsection{Proof of Theorem~\ref{thmA:generalization_error}}
\begin{thm}[Relation between proposed indicators and generalization error bound]
\label{thmA:generalization_error}
Assume the prior over the parameters $\vect{w}$ is $N(\vect{0}, \sigma^2 \mathbf{I})$. Given the constraint set $S=\{\vect{a}\in \mathbb{R}^k:\|\vect{a}\|_2= \epsilon\}$ and we choose the expectation error rate as the loss function, with probability 1-$\delta$ over the choice of the training set $\mathcal{D}\sim \mathcal{D}_1$, when $\mathcal{L}(\vect{w}, \mathcal{D}_1)$ is convex in the neighborhood of $\vect{w}$, \footnote{Note that $\mathcal{L}$ is only required to be convex in the neighbourhood of $\vect{w}$ instead of the entire $\mathbb{R}^k$.} the following generalization error bound holds,
\begin{align}
\mathcal{L}(\vect{w}, \mathcal{D}_1)\le
\mathcal{L}(\vect{w}, \mathcal{D})+
\Delta_\text{ave} \mathcal{L}(\vect{w}, S, \mathcal{D})+\mathcal{R},
\end{align}
where $R=\sqrt{\frac{C+\log\frac{|\mathcal{D}|}{\delta}}{2(|\mathcal{D}|-1)}}+o(\epsilon^2), C=\frac{\epsilon^2+\|\vect{w}\|^2_2}{2\sigma^2}-\frac{k}{2}+\frac{k}{2}\log\frac{k\sigma^2}{\epsilon^2}$ is not determined by $\mathcal{|D|}$ and $\delta$.

Generally, when $S_1=\{\vect{a}\in\mathbb{R}^k:\|\vect{a}\|_p\le\epsilon\}$, we have,
\begin{align}
\mathcal{L}(\vect{w}, \mathcal{D}_1)\le
\mathcal{L}(\vect{w}, \mathcal{D})+
\Delta_\text{max} \mathcal{L}(\vect{w}, S_1, \mathcal{D})+\mathcal{R}_1,
\end{align}
where $R_1=\sqrt{\frac{C_1+\log\frac{|\mathcal{D}|}{\delta}}{2(|\mathcal{D}|-1)}}+o(\epsilon^2),C_1=\frac{\epsilon^2+\beta_p^2\|\vect{w}\|^2_2}{2\beta_p^2\sigma^2}-\frac{k}{2}+\frac{k}{2}\log\frac{k\sigma^2\beta_p^2}{\epsilon^2}$ is not determined by $\mathcal{|D|}$ and $\delta$, here $\beta_p=\max\{1, k^{1/p-1/2}\}$.
\end{thm}
\begin{proof}

First, we introduce Lemma~\ref{lemma:Bayes}.

\begin{lem}
The following bound holds for any prior $P$ and posterior $Q$ over parameters with probability 1-$\delta$,
\begin{align}
\mathbb{E}_{\vect{w}\sim Q}[
\mathcal{L}(\vect{w}, \mathcal{D}_1)]\le
\mathbb{E}_{\vect{w}\sim Q}[
\mathcal{L}(\vect{w}, \mathcal{D})]+\sqrt{\frac{\text{KL}(Q||P)+\log\frac{|\mathcal{D}|}{\delta}}{2(|\mathcal{D}|-1)}}.
\end{align}
\label{lemma:Bayes}
\end{lem}

In Lemma~\ref{lemma:Bayes}, when $Q=N(\vect{w}, \frac{\epsilon^2}{k}\mathbf{I})$ and $P=N(0,\sigma^2)$, we have:
\begin{align}
\text{KL}(Q||P)=\frac{\epsilon^2+\|\vect{w}\|^2_2}{2\sigma^2}-\frac{k}{2}+\frac{k}{2}\log\frac{k\sigma^2}{\epsilon^2},
\end{align}
where $\text{KL}(Q||P)$ is not determined by $|\mathcal{D}|$ and $\delta$. Let $C$ be $\text{KL}(Q||P)=\frac{\epsilon^2+\|\vect{w}\|^2_2}{2\sigma^2}-\frac{k}{2}+\frac{k}{2}\log\frac{k\sigma^2}{\epsilon^2}$.

Note that, 
\begin{align}
\mathbb{E}_{\vect{w}\sim Q}[\mathcal{L}(\vect{w}, \mathcal{D})]
&=\mathbb{E}_{\vect{a}\sim N(\vect{0},\frac{\epsilon^2}{k}\mathbf{I})}[\mathcal{L}(\vect{w}+\vect{a}, \mathcal{D})]
=\mathcal{L}(\vect{w}, \mathcal{D})+\mathbb{E}_{\vect{a}\sim N(\vect{0},\frac{\epsilon^2}{k}\mathbf{I})}[\Delta\mathcal{L}(\vect{w},\vect{a}, \mathcal{D})]\\
&=\mathcal{L}(\vect{w}, \mathcal{D})+\mathbb{E}_{\vect{a}\sim N(\vect{0},\frac{\epsilon^2}{k}\mathbf{I})}[\vect{a}^\text{T}\vect{g}+\frac{1}{2}\vect{a}^\text{T}\textbf{H}\vect{a}]+o(\epsilon^2).
\end{align}

According to Theorem~\ref{thm:random} and $S=\{\vect{a}\in \mathbb{R}^k:\|\vect{a}\|_2= \epsilon\}$, we have,
\begin{align}
\mathbb{E}_{\vect{a}\sim N(\vect{0},\frac{\epsilon^2}{k}\mathbf{I})}[\vect{a}^\text{T}\vect{g}+\frac{1}{2}\vect{a}^\text{T}\textbf{H}\vect{a}]+o(\epsilon^2)
&=\mathbb{E}_{a_i\sim N(0, \frac{\epsilon^2}{k})}[\sum\limits_{i}g_ia_i]+\mathbb{E}_{a_i\sim N(0, \frac{\epsilon^2}{k})}[\frac{1}{2}\sum\limits_{i, j}H_{ij}a_ia_j]+o(\epsilon^2)\\
&=\sum\limits_{i}H_{ii}\frac{\epsilon^2}{2k}+o(\epsilon^2)=\frac{\text{trace}(\textbf{H})}{2k}\epsilon^2+o(\epsilon^2)\\
&=\Delta_\text{ave}\mathcal{L}(\vect{w}, S,\mathcal{D})+o(\epsilon^2).
\end{align}

Because the loss function is convex in the neighborhood, according to Jensen inequality,
\begin{align} 
\mathbb{E}_{\vect{w}\sim Q}[\mathcal{L}(\vect{w}, \mathcal{D}_1)]\ge \mathcal{L}(\mathbb{E}_{\vect{w}\sim Q}[\vect{w}], \mathcal{D}_1)=\mathcal{L}(\vect{w}, \mathcal{D}_1).
\end{align}

Therefore, 
\begin{align}
\mathcal{L}(\vect{w}, \mathcal{D}_1)\le
\mathcal{L}(\vect{w}, \mathcal{D})+
\Delta_\text{ave} \mathcal{L}(\vect{w}, S, \mathcal{D})+\sqrt{\frac{C+\log\frac{|\mathcal{D}|}{\delta}}{2(|\mathcal{D}|-1)}}+o(\epsilon^2),
\label{eqA:bound1}
\end{align}
where $C=\frac{\epsilon^2+\|\vect{w}\|^2_2}{2\sigma^2}-\frac{k}{2}+\frac{k}{2}\log\frac{k\sigma^2}{\epsilon^2}$ is not determined by $\mathcal{|D|}$ and $\delta$.

For the second conclusion, we introduce Lemma~\ref{lemma:norm2}. Lemma~\ref{lemma:norm2} and its proof is similar to Lemma~\ref{lemma:norm1}.

\begin{lem}
For vector $\vect{x}\in \mathbb{R}^k$, for any $p>1$, $\|\vect{x}\|_p\le \beta_p \|\vect{x}\|_2$, where $\beta_p=\max\{1, k^{1/p-1/2}\}$. 
\label{lemma:norm2}
\end{lem}

\begin{proof}[Proof of Lemma~\ref{lemma:norm2}]

When $p>2$, define $t=\frac{p}{2}>1$ and $h(x)=x^t+(1-x)^t$, $h''(x)=t(t-1)(x^{t-2}+(1-x)^{t-2})>0$, thus $h(x)\le \max\{h(0), h(1)\}=1\ (x\in[0, 1])$. 

Then for $a, b\ge 0$ and $a+b>0$, we have $\frac{a^t+b^t}{(a+b)^t}=(\frac{a}{a+b})^t+(1-\frac{a}{a+b})^t=h(\frac{a}{a+b})\le 1$. That is to say, $a^t+b^t \le (a+b)^t$. More generally, $a^t+b^t+\cdots+c^t \le (a+b+\cdots+c)^t$. Therefore,
\begin{align}
\|\vect{x}\|_p=\big(\sum\limits_{i=1}^k|x_i|^{p}\big)^{\frac{1}{p}}=\big(\sum\limits_{i=1}^k(|x_i|^2)^\frac{p}{2}\big)^{\frac{1}{p}} \le \big((\sum\limits_{i=1}^k|x_i|^2)^\frac{p}{2}\big)^{\frac{1}{p}} = \|\vect{x}\|_2.
\end{align}

When $p\le 2$, according to the power mean inequality,
\begin{align}
\|\vect{x}\|_p=\big(\sum\limits_{i=1}^k|x_i|^{p}\big)^{\frac{1}{p}}=k^\frac{1}{p} \big(\frac{\sum\limits_{i=1}^k|x_i|^p}{k}\big)^{\frac{1}{p}} \le k^\frac{1}{p} \big(\frac{\sum\limits_{i=1}^k|x_i|^2}{k}\big)^{\frac{1}{2}}=  k^{\frac{1}{p}-\frac{1}{2}}(\sum\limits_{i=1}^k|x_i|^2)^\frac{1}{2} = k^{\frac{1}{p}-\frac{1}{2}}\|\vect{x}\|_2.
\end{align}

To conclude, $\|\vect{x}\|_p\le \beta_p \|\vect{x}\|_2$, where $\beta_p=\max\{1, k^{1/p-1/2}\}$.
\end{proof}

According to Lemma~\ref{lemma:norm2}, for any vector $\vect{a}$, if $\|\vect{a}\|_2\le{\epsilon}/{\beta_p}$, then $\|\vect{a}\|_p \le \beta_p\|\vect{a}\|_2=\epsilon$. Therefore, $S_3=\{\vect{a}:\|\vect{a}\|_2= {\epsilon}/{\beta_p}\} \subset S_2=\{\vect{a}:\|\vect{a}\|_2\le {\epsilon}/{\beta_p}\}\subset S_1=\{\vect{a}:\|\vect{a}\|_p\le \epsilon\}$.

According to Eq.(\ref{eqA:bound1}),
\begin{align}
\mathcal{L}(\vect{w}, \mathcal{D}_1)\le
\mathcal{L}(\vect{w}, \mathcal{D})+
\Delta_\text{ave} \mathcal{L}(\vect{w}, S_3, \mathcal{D})+\sqrt{\frac{C_1+\log\frac{|\mathcal{D}|}{\delta}}{2(|\mathcal{D}|-1)}}+o(\epsilon^2),
\end{align}
where $C_1=\frac{\epsilon^2+\beta_p^2\|\vect{w}\|^2_2}{2\beta_p^2\sigma^2}-\frac{k}{2}+\frac{k}{2}\log\frac{k\sigma^2\beta_p^2}{\epsilon^2}$ is not determined by $\mathcal{|D|}$ and $\delta$, here $\beta_p=\max\{1, k^{1/p-1/2}\}$.

Note that,
\begin{align}
\Delta_\text{ave} \mathcal{L}(\vect{w}, S_3, \mathcal{D}) \le \Delta_\text{max} \mathcal{L}(\vect{w}, S_3, \mathcal{D}) \le \Delta_\text{max} \mathcal{L}(\vect{w}, S_2, \mathcal{D}) \le \Delta_\text{max} \mathcal{L}(\vect{w}, S_1, \mathcal{D}).
\end{align}

Therefore,
\begin{align}
\mathcal{L}(\vect{w}, \mathcal{D}_1)\le
\mathcal{L}(\vect{w}, \mathcal{D})+
\Delta_\text{max} \mathcal{L}(\vect{w}, S_1, \mathcal{D})+\sqrt{\frac{C_1+\log\frac{|\mathcal{D}|}{\delta}}{2(|\mathcal{D}|-1)}}+o(\epsilon^2).
\end{align}
\end{proof}

\subsection{Closed-form Solutions in Corruption}

The close-form solutions of the gradient-based corruption can be generalized into Proposition~\ref{prop:linear}, which is the maximum of linear function under the corruption constraint. We also provide closed-form solution in the multi-step corruption in Proposition~\ref{prop:update} and Proposition~\ref{prop:project}.

\begin{prop}[Constrained Maximum]
\label{prop:linear}
Given a vector $\vect{v}\in \mathbb{R}^k$, the optimal $\vect{\hat a}$ that maximizes $\vect{a}^\text{T}\vect{v}$ under the corruption constraint $\vect{a}\in S=\{\vect{a}:\|\vect{a}\|_p= \epsilon\text{ and }\|\vect{a}\|_0\le n\}$ is:
\begin{equation}
\vect{\hat a}=\argmax_{\vect{a}\in S}\vect{a}^\text{T}\vect{v}=\epsilon (\text{sgn}(\vect{h})\odot\frac{|\vect{h}|^\frac{1}{p-1}}{\||\vect{h}|^\frac{1}{p-1}\|_p}),\text{ and}\quad \vect{\hat a}^\text{T}\vect{v}=\epsilon\|\vect{h}\|_{\frac{p}{p-1}},
\end{equation}
where $\vect{h}=\text{top}_n(\vect{v})$, retaining top-$n$ magnitude of all $|\vect{v}|$ dimensions and set other dimensions to $0$, $\text{sgn}(\cdot)$ denotes the signum function, $|\cdot|$ denotes the point-wise absolute function, and $(\cdot)^\alpha$ denotes the point-wise $\alpha$-power function.
\end{prop}

\begin{proof}
When $\vect{a}\in S=\{\vect{a}:\|\vect{a}\|_p=\epsilon\text{ and }\|\vect{a}\|_0\le n\}$, define $\vect{a}=\textbf{P}\vect{b}$, where $\textbf{P}$ is a diagonal $0/1$ matrix with $n$ ones. It is easy to verify $\textbf{P}^\text{T}=\textbf{P}=\textbf{P}^2$. Define $q=\frac{p}{p-1}, \frac{1}{p}+\frac{1}{q}=1$ here. Then according to Holder Inequality, for $\frac{1}{p}+\frac{1}{q}=1,(1\le p, q\le+\infty)$,
\begin{align}
\vect{a}^\text{T}\vect{v}=\vect{b}^\text{T}\textbf{P}\vect{v}=\vect{b}^\text{T}\textbf{P}\textbf{P}\vect{v}=\vect{a}^\text{T}(\textbf{P}\vect{v})\le\|\vect{a}\|_p\|\textbf{P}\vect{v}\|_q=\epsilon\|\vect{h}\|_{\frac{p}{p-1}},
\end{align}
where $\vect{h}=\textbf{M}\vect{v}=\text{top}_n(\vect{v})$, $\textbf{M}$ is a diagonal $0/1$ matrix and $M_{j,j}=1$ if and only if $|\vect{v}|_j$ is in the top-$n$ magnitude of all $|\vect{v}|$ dimensions. The equation holds if and only if,
\begin{align}
\vect{\hat a}=\epsilon(\text{sgn}(\vect{h})\odot\frac{|\vect{h}|^{\frac{1}{p-1}}}{\||\vect{h}|^{\frac{1}{p-1}}\|_p}),
\end{align}
and the maximum value of $\vect{a}^\text{T}\vect{v}$ is $\vect{\hat a}^\text{T}\vect{v}=\epsilon\|\vect{h}\|_{\frac{p}{p-1}}$.
\end{proof}

\begin{prop}
\label{prop:update}
When maximizing $\vect{u}^\text{T}\vect{g}$ under the constraint $\|\vect{u}\|_p=\alpha$, solutions to $L_2$ and $L_{+\infty}$ cases are:
\begin{equation}
\argmax\limits_{\|\vect{u}\|_2=\alpha}\vect{u}^\text{T}\vect{g}=\alpha\frac{\vect{g}}{\|\vect{g}\|_2};\quad 
\argmax\limits_{\|\vect{u}\|_{+\infty}=\alpha}\vect{u}^\text{T}\vect{g}=\alpha\text{sgn}(\vect{g}),
\end{equation}
where $\text{sgn}(\cdot)$ denotes the signum function.
\end{prop}

\begin{proof}

First, let us consider a general case. According to Holder Inequality, for $\frac{1}{p}+\frac{1}{q}=1,(1\le p, q\le+\infty)$,
\begin{align}
\vect{u}^\text{T}\vect{g}\le\|\vect{u}\|_p\|\vect{g}\|_q=\alpha\|\vect{h}\|_{\frac{p}{p-1}}.
\end{align}

The equation holds if and only if,
\begin{align}
\vect{u}=\alpha (\text{sgn}(\vect{g})\odot\frac{|\vect{g}|^{\frac{1}{p-1}}}{\||\vect{g}|^{\frac{1}{p-1}}\|_p}).
\end{align}

When $p=2$,
\begin{align}
\vect{u}=\alpha\text{sgn}(\vect{g}) \frac{|\vect{g}|}{\|\vect{g}\|_2}=\alpha \frac{\vect{g}}{\|\vect{g}\|_2}.
\end{align}

When $p\to +\infty$, $0^{\frac{1}{p-1}}\to 0, x^{\frac{1}{p-1}}\to 1\ (x\ne 0)$ and $|\vect{g}|^{\frac{1}{p-1}}\to \mathbb{I}(\vect{g}\ne 0)$. Then:
\begin{align}
\vect{u}=\lim\limits_{p\to +\infty }\alpha (\text{sgn}(\vect{g})\odot\frac{|\vect{g}|^{\frac{1}{p-1}}}{\||\vect{g}|^{\frac{1}{p-1}}\|_p})=\alpha \text{sgn}(\vect{g}).
\end{align}
\end{proof}

\begin{prop}
\label{prop:project}
When minimizing $\|\vect{y}-\vect{x}\|_2$ under the constraint $\|\vect{y}\|_p\le \epsilon, \|\vect{y}\|_0\le n$, solutions to $L_2$ and $L_{+\infty}$ cases are:
\begin{equation}
\argmax\limits_{\|\vect{y}\|_2\le \epsilon, \|\vect{y}\|_0\le n}\|\vect{y}-\vect{x}\|_2=\min\{\|\vect{h}\|_2, \epsilon\}\frac{\vect{h}}{\|\vect{h}\|_2};\quad 
\argmax\limits_{\|\vect{y}\|_{+\infty}\le \epsilon, \|\vect{y}\|_0\le n}\|\vect{y}-\vect{x}\|_2=\text{clip}(\vect{h}, -\epsilon, \epsilon),
\end{equation}
where $\vect{h}=\text{top}_n(\vect{x})$, which only retains top-$n$ magnitude of all $|\vect{x}|$ dimensions and set other dimensions to $0$. $\text{clip}(\vect{h}, -\epsilon, \epsilon)$ clip every dimensions of $\vect{h}$ into $[-\epsilon, \epsilon]$.
\end{prop}

\begin{proof}
We may assume $|\vect{x}_i|\ne |\vect{x}_j|$ for every $i\ne j$, thus $\text{top}_n$ dimensions of $\vect{x}$ are unique.

Suppose we choose first $n$ dimensions of $\vect{y}$ as possible non-zero dimensions, namely, $\vect{y}_{n+1}=\vect{y}_{n+2}=\cdots=\vect{y}_k=0$, $\vect{z}=[\vect{x}_1, \vect{x}_2, \cdots, \vect{x}_n, 0, 0, \cdots, 0]^\text{T}$, then $(\vect{x}-\vect{z})\perp (\vect{y}-\vect{z})$ and $(\vect{x}-\vect{z})\perp \vect{z}$,
\begin{align}
\|\vect{y}-\vect{x}\|_2^2=\|(\vect{y}-\vect{z})-(\vect{x}-\vect{z})\|_2^2=\|\vect{y}-\vect{z}\|_2^2+\|\vect{x}-\vect{z}\|_2^2.
\end{align}

Define $\vect{h}=\text{top}_n(\vect{x})$, we have $\|\vect{z}\|_2\le\|\vect{h}\|_2$. We will prove that $\|\vect{y}-\vect{x}\|_2$ is equal to the minimum value if and only if the $n$ dimensions we choose are $n$ dimensions with $\text{top}_n$ magnitude, namely, $\vect{z}=\vect{h}$.

First, let us consider the case when $p=2$:

(1) When $\|\vect{z}\|\le\|\vect{h}\|\le \epsilon$,
\begin{align}
\|\vect{y}-\vect{x}\|_2^2=\|(\vect{y}-\vect{z})-(\vect{x}-\vect{z})\|_2^2=\|\vect{y}-\vect{z}\|_2^2+\|\vect{x}-\vect{z}\|_2^2\ge \|\vect{x}-\vect{z}\|_2^2\ge \|\vect{x}-\vect{h}\|_2^2.
\end{align}

The inequality holds if and only if $\vect{y}=\vect{z}=\vect{h}$ here.

(2) When $\|\vect{h}\|\ge \|\vect{z}\|\ge \epsilon$,
\begin{align}
\|\vect{y}-\vect{x}\|_2^2=&\|(\vect{y}-\vect{z})-(\vect{x}-\vect{z})\|_2^2=\|\vect{y}-\vect{z}\|_2^2+\|\vect{x}-\vect{z}\|_2^2\ge (\|\vect{z}\|_2-\|\vect{y}\|_2)^2+\|\vect{x}-\vect{z}\|_2^2 \\
\ge& (\|\vect{z}\|_2-\epsilon)^2+\|\vect{x}-\vect{z}\|_2^2 = \|\vect{z}\|_2^2-2\epsilon\|\vect{z}\|_2+\epsilon^2+\|\vect{x}-\vect{z}\|_2^2\\
=& -2\epsilon\|\vect{z}\|_2+\epsilon^2+\|\vect{x}\|_2^2\ge-2\epsilon\|\vect{h}\|_2+\epsilon^2+\|\vect{x}\|_2^2.
\end{align}

The inequality holds if and only if $\vect{y}=\nicefrac{\epsilon\vect{z}}{\|\vect{z}\|_2}=\nicefrac{\epsilon\vect{h}}{\|\vect{h}\|_2}$ here.

(3) When $\|\vect{h}\|\ge \epsilon > \|\vect{z}\|$,
\begin{align}
\|\vect{y}-\vect{x}\|_2^2=&\|(\vect{y}-\vect{z})-(\vect{x}-\vect{z})\|_2^2=\|\vect{y}-\vect{z}\|_2^2+\|\vect{x}-\vect{z}\|_2^2\ge \|\vect{x}-\vect{z}\|_2^2 \\
=& \|\vect{x}\|_2^2-\|\vect{z}\|_2^2 > \|\vect{x}\|_2^2-\epsilon^2 = -2\epsilon^2+\epsilon^2+\|\vect{x}\|_2^2 \ge  -2\epsilon\|\vect{h}\|^2_2+\epsilon^2+\|\vect{x}\|_2^2.
\end{align}

We can see, under these circumstances, $\|\vect{y}-\vect{x}\|_2^2$ is larger than the minimum of (2). Therefore, the minimum of $\vect{y}$ will not be in (3).

To conclude, $\argmax\limits_{\|\vect{y}\|_2\le \epsilon, \|\vect{y}\|_0\le n}\|\vect{y}-\vect{x}\|_2=\min\{\|\vect{h}\|_2, \epsilon\}\frac{\vect{h}}{\|\vect{h}\|_2}$.

Then, let us consider the case when $p=+\infty$:
To make $\|\vect{y}-\vect{z}\|_2$ minimal, we should choose  $\vect{y}_i=\text{clip}(\vect{x}_i, -\epsilon, \epsilon), (i\le n)$ and then: 
\begin{align}
\|\vect{y}-\vect{x}\|_2^2=&\|\vect{x}\|^2-2\vect{x}^\text{T}\vect{y}+\|\vect{y}\|_2^2=\|\vect{x}\|^2+(\vect{y}-2\vect{x})^\text{T}\vect{y}=\|\vect{x}\|^2+\sum\limits_{i=1}^n(\vect{y}_i-2\vect{x}_i)\vect{y}_i.
\end{align}

Consider $f(t)=(\text{clip}(t, -\epsilon, \epsilon)-2t)\text{clip}(t, -\epsilon, \epsilon)$. It is easy to verify $f(t)=f(-t)$, thus we may assume $t\ge 0$. Then $f(t)=(\epsilon-2t)\epsilon$ when $t>\epsilon$ or $f(t)=-t^2$ when $t\le \epsilon$. $f(t)$ is monotonically decreasing when $t>0$. Therefore, $(\vect{y}_i-2\vect{x}_i)\vect{y}_i>(\vect{y}_j-2\vect{x}_j)\vect{y}_j$ if $|\vect{x}_i|<|\vect{x}_j|$ and $i,j\le n$.

To make $\|\vect{x}-\vect{z}\|_2$ minimal, we should choose $\text{top}_n$ dimensions, namely, $\vect{z}=\vect{h}$ and $\vect{y}=\text{clip}(\vect{z}, -\epsilon, \epsilon)$.

To conclude, $\argmax\limits_{\|\vect{y}\|_{+\infty}\le \epsilon, \|\vect{y}\|_0\le n}\|\vect{y}-\vect{x}\|_2=\text{clip}(\vect{h}, -\epsilon, \epsilon)$.
\end{proof}
 
\end{document}